\documentclass{article}

\usepackage[preprint]{neurips_2026}
\usepackage{amsmath}
\usepackage{amssymb}
\usepackage{mathtools}
\usepackage{amsthm}
\usepackage{microtype}
\usepackage{graphicx}
\usepackage{subcaption}
\usepackage{booktabs} % for professional tables

\usepackage{mathrsfs}
\usepackage{amssymb}
\usepackage{bm}
\usepackage{amsmath}
\usepackage{dsfont}
\usepackage{booktabs}
\usepackage{epstopdf}
\usepackage{multirow}
\usepackage{comment}
\usepackage{endnotes}
\usepackage{hyperref}

\usepackage{bibentry}
\usepackage[normalem]{ulem}
\usepackage{graphicx}
\usepackage{enumerate}
\usepackage{enumitem}
\usepackage{subcaption}
\usepackage{etoc}
\usepackage{xcolor,xspace}
\definecolor{mydarkred}{rgb}{0.6,0,0}
\definecolor{myblue}{HTML}{268BD2}
\usepackage{inconsolata}
\usepackage[table]{xcolor}
\usepackage{hyperref}
\hypersetup{
  colorlinks=true,
  linkcolor=mydarkred,
  citecolor=myblue
}
\usepackage{multirow}
\usepackage[capitalize,noabbrev]{cleveref}
\newcommand{\ourmodel}{{\texttt{\textbf{MemoType}}}\xspace}
\newcommand{\ourdata}{{\texttt{\textbf{TriMEM}}}\xspace}
\theoremstyle{plain}
\newtheorem{theorem}{Theorem}[section]

\theoremstyle{definition}

\newtheorem{assumption}[theorem]{Assumption}
\theoremstyle{remark}
\newtheorem{remark}[theorem]{Remark}

\usepackage[utf8]{inputenc} % allow utf-8 input
\usepackage[T1]{fontenc}    % use 8-bit T1 fonts
\usepackage{hyperref}       % hyperlinks
\usepackage{url}            % simple URL typesetting
\usepackage{booktabs}       % professional-quality tables
\usepackage{amsfonts}       % blackboard math symbols
\usepackage{nicefrac}       % compact symbols for 1/2, etc.
\usepackage{microtype}      % microtypography
\usepackage{xcolor}         % colors
\usepackage{graphicx}
\usepackage{hyperref}
\usepackage{wrapfig}
\usepackage{thmtools}
\usepackage{thm-restate}

\usepackage{tcolorbox}
\usepackage{xcolor}
\usepackage{fontawesome5}
\usepackage{enumitem}
\tcbuselibrary{skins, breakable}
 
\definecolor{promptbg}{HTML}{F5F7FA}
\definecolor{promptframe}{HTML}{3B6FA0}
\definecolor{prompttitle}{HTML}{FFFFFF}
 
\newtcolorbox{promptbox}[1]{
    enhanced,
    breakable,
    colback=promptbg,
    colframe=promptframe,
    colbacktitle=promptframe,
    coltitle=prompttitle,
    fonttitle=\bfseries\sffamily,
    title={\faRobot\quad #1},
    arc=3mm,
    boxrule=0.8pt,
    left=10pt, right=10pt, top=8pt, bottom=8pt,
    toptitle=4pt, bottomtitle=4pt,
    shadow={1mm}{-1mm}{0mm}{black!15},
    attach boxed title to top left={yshift=-2mm, xshift=5mm},
    boxed title style={sharp corners, size=small}
}
\newcommand{\phl}[1]{\textcolor{promptframe}{#1}}
 
\newtcolorbox{promptcode}{
    enhanced,
    colback=white,
    colframe=promptframe!30,
    boxrule=0.4pt,
    arc=1.5mm,
    left=8pt, right=8pt, top=4pt, bottom=4pt
}
 
\newenvironment{promptrules}{
    \begin{enumerate}[leftmargin=1.5em, itemsep=2pt, label=\textcolor{promptframe}{\arabic*.}]
}{
    \end{enumerate}
}

\title{Memory Type Varies: Empowering LLM Agents for Long-Term Memory with Diverse Strategies}

\author{%
  Yi Wen$^{1}$, Derong Xu$^{1}$, Pengyue Jia$^{1}$, Yichao Wang$^{2}$ , Yingyi Zhang$^{1}$, Maolin Wang$^{1}$, \\
  \textbf{Junyi Li$^{1}$, Wenlin Zhang$^{1}$, Xiaopeng Li$^{1}$, Yong Liu$^{2}$, Xiangyu Zhao$^{1}$} \\
  $^{1}$City University of Hong Kong,  $^{2}$Huawei Noah's Ark Lab \\
  \texttt{yiwen23-c@my.cityu.edu.hk}, 
  \texttt{wangyichao5@huawei.com}, 
  \texttt{xianzhao@cityu.edu.hk}
}

\begin{document}

\maketitle

\begin{abstract}
 The memory capabilities of Large Language Models (LLMs) have garnered increasing attention recently. Despite great success achieved, existing retrieval-based memory approaches typically overlook the differences between memories and employ a unified strategy to process all memories, leading to suboptimal performance. Thus, an intuitive question arises: can we categorize memory into different types and select appropriate strategies? However, given the topic-rich, scenario-complex, and boundary-blurred nature of memory scenarios, achieving precise classification of memories is not easy. To address this challenge, we propose a memory multi-class dataset in this paper, termed \ourdata, which provides precise annotations for memory types across diverse scenarios. Building upon this foundation, we propose a novel memory framework, named \ourmodel, which can adaptively recognize each memory and query type with the learned router model. With the memory and query routing, MemoType can retrieve the memory with corresponding query types and design tailored retrieval strategies, thereby enhancing the retrieval performance. Moreover, we theoretically prove that any single retrieval strategy is subject to a fundamental upper bound on its expected query-document alignment in multi-class corpora, leading to systematic performance degradation. Extensive experiments on three datasets demonstrate that MemoType consistently outperforms existing methods, achieving up to 16.18\% improvement in Recall@1\footnote{ \url{https://github.com/wenyiwy99/MemoType}.}.
\end{abstract}

\section{Introduction}

    Large Language Models (LLMs) have demonstrated remarkable abilities in natural language understanding and multi-turn dialogue fields, garnering increasing attention. However, with the growing frequency of User-LLM interactions, the limited context window length of LLMs fails to meet the demand for scalable conversational interactions~\citep{tim}. When presented with retrospective queries, LLMs cannot effectively utilize conversation memories to achieve personalized responses, resulting in the ``Forgetting'' problem. To address this issue, Retrieval-Augmented Generation (RAG)-based memory techniques have emerged as a practical solution. Instead of concatenating all conversations, RAG-based approaches store conversations in an external bank for selective retrieval, which enables the LLM to recall critical memories and achieve personalized responses.

    RAG-based memory strategies~\citep{a-mem, memochat} can be broadly classified into unstructured and structured types, based on whether to construct additional memory associative information. Unstructured memory augmentation approaches typically store dialogue fragments or summaries, retrieved through similarity matching with the given query. An early and influential effort is the MemoryBank~\citep{memorybank}, which maintains a user-specific memory and retrieves facts with relevance scoring to produce appropriate responses. On this basis, SeCom~\citep{secom} employs adaptive memory segmentation techniques to achieve superior performance. Structural memory augmentation strategies organize memories into meaningful relationships, thereby achieving stronger sense-making and richer associations. For instance, HippoRAG~\citep{hipporag} builds a graph and uses personalized pageRank to surface associative passages. Structural memory augmentation strategies deepen memory connections while inevitably introducing additional overhead.

    Despite great success, the above methods still suffer certain drawbacks. In the retrieval phase, most approaches overlook the structural and content differences between memories and employ a unified strategy to process all memories, leading to suboptimal performance. For instance, event-related memories often contain fragmented information such as persons, time, and location, preventing regular embedding-based similarity from achieving precise retrieval. Conversely, when memories pertain only to basic concepts, introducing details like time or persons may introduce noise into the retrieval process. Thus, an intuitive question arises: 
\definecolor{cvprblue}{rgb}{0.21,0.49,0.74}
\begin{tcolorbox}[colframe=black!50, colback=cvprblue!8, boxrule=1.5pt, arc=2mm, top=4pt, bottom=4pt, left=4pt, right=4pt,  boxsep=1pt]
\raisebox{-0.2\baselineskip}{\includegraphics[height=1\baselineskip]{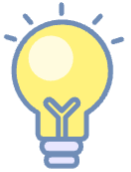}}
\textit{Can we categorize memory into different types and select tailored strategies for each type?}
\end{tcolorbox}

However, addressing the aforementioned issues presents two major obstacles: i) \textit{Existing memory systems lack a unified classification framework}. Developing a comprehensive scheme that adequately captures the diversity of real conversation scenarios remains a significant challenge.\footnote{In this paper, the memory we discuss primarily refers to dialogue memory, which mainly encompasses two categories: user-to-user conversations and user-to-assistant conversations. This differs from the experience memory in agent fields. We will not repeat this explanation in subsequent sections.}. ii) \textit{Precise classification is inherently difficult given the topic-rich, scenario-complex, and boundary-blurred nature of memory scenarios}. Although some attempts have been conducted, most research~\citep{mms,mirix} relies on human-crafted categories or adopts basic prompt-based methods to determine memory types, making it difficult to adapt to complex memory scenarios to achieve accurate classification, thus failing to acquire optimal results. 

\begin{figure*}[t]
    \centering
    \begin{subfigure}[t]{0.45\linewidth}
        \centering
        \includegraphics[width=\linewidth]{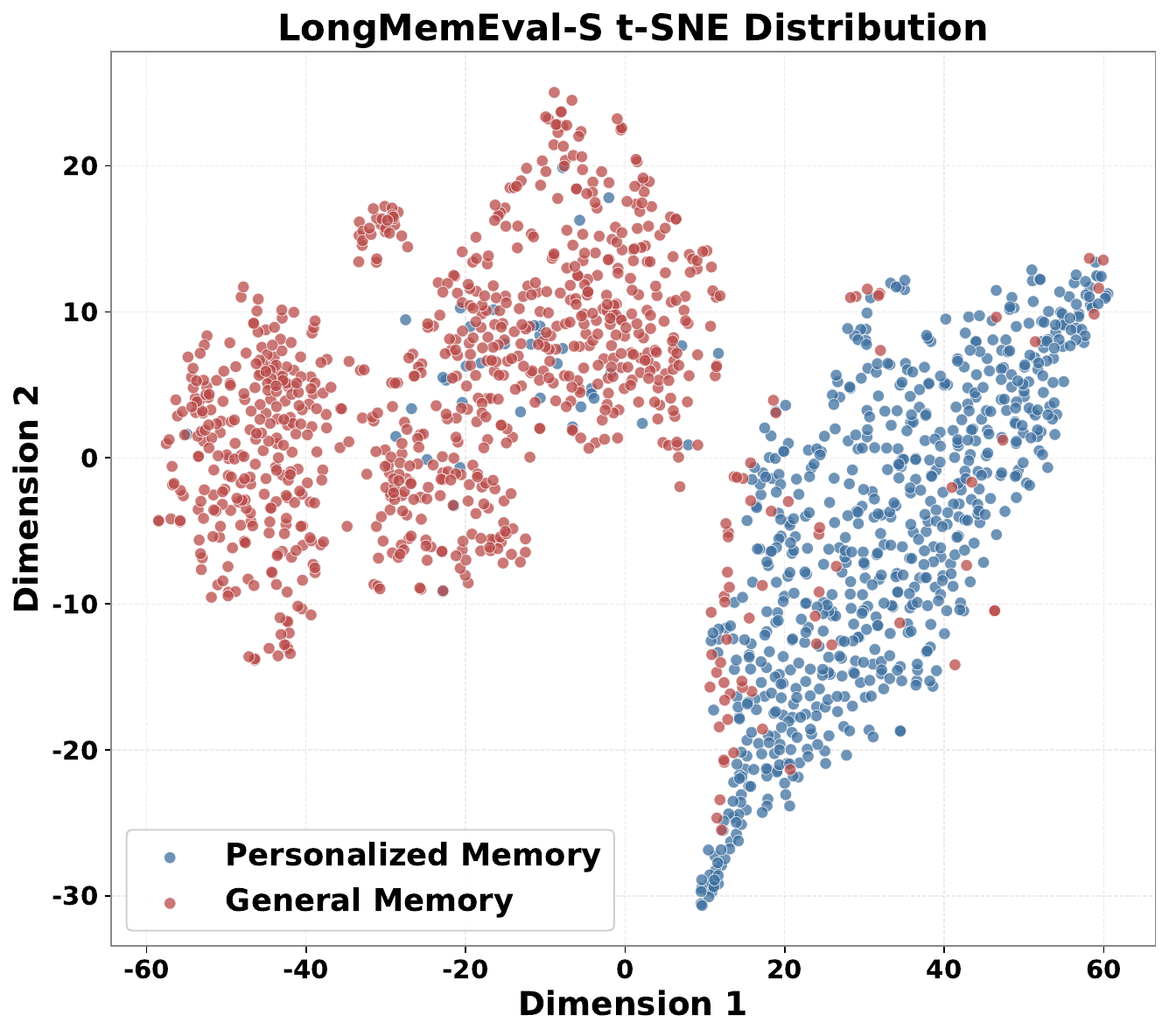}
        \caption{The t-SNE distribution of LongMemEval-S.}
        \label{fig:longmem_s_tsne}
    \end{subfigure}
    \begin{subfigure}[t]{0.5\linewidth}
        \centering
        \setlength{\abovecaptionskip}{5.5pt}
        \includegraphics[width=\linewidth]{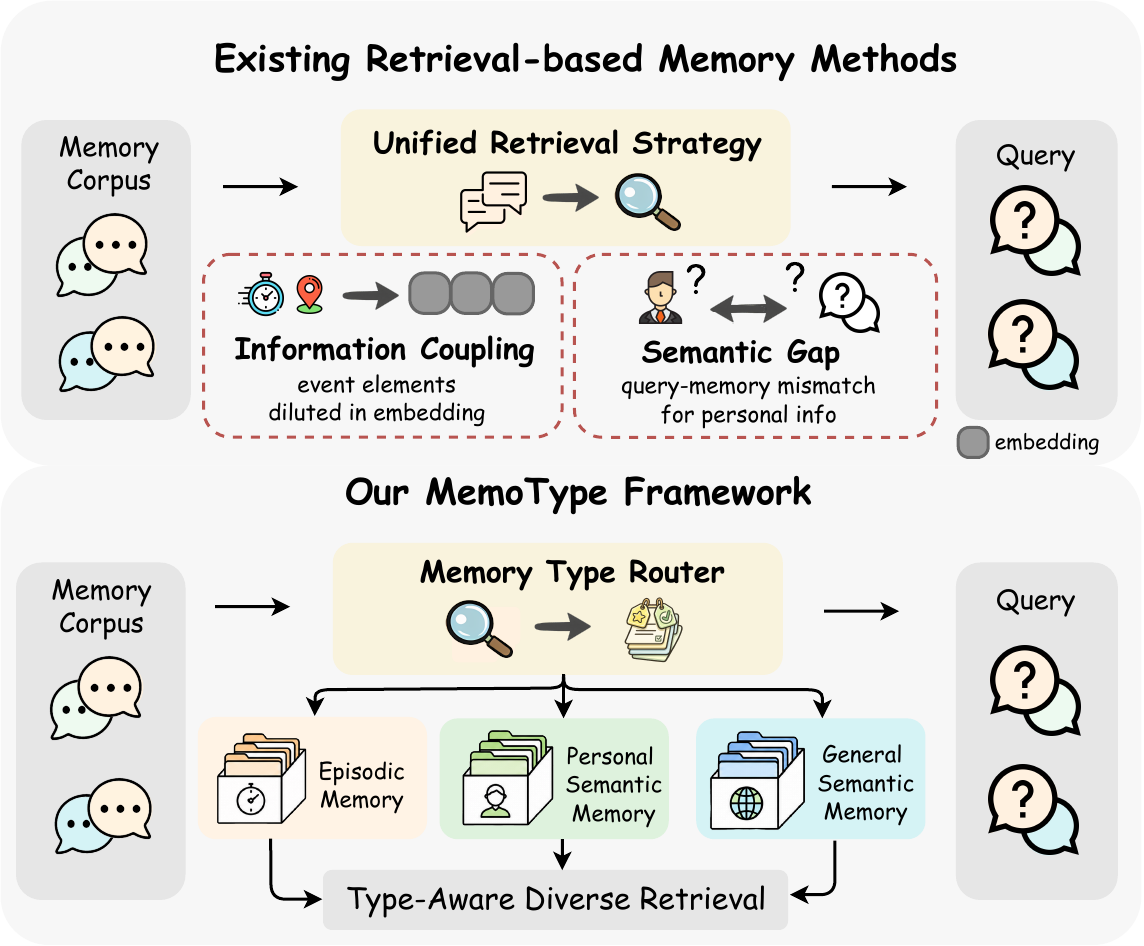}
        \caption{Existing unified retrieval vs. MemoType framework.   }
        \label{fig:motivation}
    \end{subfigure}
    \caption{Evidence and Motivation: Semantic Clustering of memory and the \ourmodel framework.}
    \vspace{-1.5 em}
\end{figure*}

% \begin{wrapfigure}{r}{0.5\textwidth}

% \centering
%     \includegraphics[width=\linewidth]{fig/LongMemEval-S_TSNE.pdf}
%     \caption{\small The t-SNE Distribution of LongMemEval-S Dataset.}%
%     \label{longmem_s_tsne}
%     \vspace{-5pt}
% \end{wrapfigure}
To tackle these challenges, we first draw on Tulving’s distinction theory~\citep{episodic_semantic} in cognitive psychology~\citep{cognitive_psychology}, which distinguishes memory into episodic and semantic components. Episodic memory records events tied to specific time, place, and participants, whereas semantic memory captures facts and stable knowledge. However, this coarse dichotomy is insufficient for real-world dialogue. Event-centric discussions are relatively rare. For instance, in the LongMemEval-S dataset, episodic memory accounts for only 12.35\% of instances, making semantic memory overwhelmingly dominant and poorly resolved. This observation motivates a finer-grained semantic classification. Surprisingly, the t-SNE visualizations of LongMemEval, as shown in Figure \ref{fig:longmem_s_tsne}, reveal a pronounced semantic separation within semantic memory: the personal content and general knowledge form two distinct clusters with a clear semantic gap. Guided by this evidence, we introduce a well-defined categorization principle consisting of Episodic Memory, Personal Semantic Memory, and General Semantic Memory.

Inspired by the above classification principle, we construct a memory multi-class dataset, termed \ourdata, a memory multi-class dataset comprising 6,000 samples with fine-grained type annotations across diverse topics and scenarios. With the help of TriMEM, we propose a novel memory framework, named \ourmodel. As shown in Figure~\ref{fig:motivation}, MemoType can adaptively recognize each memory and query type with the learned router model. With the memory and query routing, MemoType can retrieve the memory with corresponding query types and design tailored retrieval strategies, thereby enhancing the retrieval in both effectiveness and efficiency. Besides, we design a memory pruning module that adaptively prunes retrieved memories based on queries, reducing irrelevant information. Moreover, we provide a theoretical analysis showing that any single retrieval strategy is subject to a fundamental upper bound on its expected query-document alignment in multi-class corpora, leading to systematic performance degradation.  Our contributions are as follows:
\vspace{-0.5 em}
\begin{itemize} [leftmargin=*]
\item Inspired by Tulving’s distinction theory, we propose a novel categorization principle and construct \ourdata, a memory multi-class dataset with accurate annotations spanning diverse topics, providing a solid foundation for memory-type classification in memory field. 
\vspace{-0.5 em}
\item We propose \ourmodel, a novel memory framework that adaptively recognizes memory and query types via a learned router, applies tailored retrieval strategies for each type, and further prunes irrelevant information from retrieved memories through a memory pruning module.
\vspace{-0.5 em}
\item We theoretically prove that any single retrieval strategy has a fundamental upper bound on its expected query-doc alignment in multi-class corpora, leading to systematic performance degradation.
\end{itemize}
\vspace{-0.5 em}

\vspace{-0.5em}
\section{MemoType}
\vspace{-0.5em}

\subsection{PRELIMINARY}
    Denote $\mathcal{M} = \{ \mathcal{S}_i \}_{i=1}^S$ as the stored conversation history between two participants, where the dialogue may occur between users or between a user and an assistant. $S$ indicates the number of sessions, $\mathcal{S}_i = \{ d_l \}_{l=1}^{S_i}$ denotes the $i$-th session, comprising $S_i$ sequential dialogue. Denote the retrieval function as $f_R$ and the response generation function as $f_{G}$. The research framework process can be described as follows: (1) \textit{Memory Construction}: Build a memory repository $\mathcal{B}$ from $\mathcal{M}$. (2) \textit{Memory Retrieval}: Given a user query $q$ and a retrieval repository $B$, retrieve $n$ memory turns $\{t_i \}_{i=1}^n \subseteq \mathcal{B}$ relevant to $q$ using the function $f_R(q, \mathcal{B}, n)$. (3) \textit{Response Generation}: Employ the retrieved memory turns $\{ t_i\}_{i=1}^n$ and user query $q$ to generate the output $r$ using $f_{G}(q, \{ t_i \}_{i=1}^n)$.

\vspace{-0.5em}

\subsection{TriMEM Dataset}
As mentioned in the previous sections, different types of memories require tailored strategies to optimize performance. A natural idea is to categorize memories into distinct types and apply appropriate strategies accordingly. However, due to the topic-rich, scenario-complex, and boundary-blurred nature of memory scenarios, achieving precise classification is a challenging task. While some preliminary~\citep{mirix, mms} efforts have been made in memory categorization, most existing approaches rely on basic prompt-based methods to determine memory types, while we verify that prompt-based methods fail to adapt to complex memory scenarios to achieve accurate classification in Table~\ref{trimem_class} in the Appendix~\ref{app:trimem_evaluation}. 

Existing memory classification efforts remain limited. The most relevant work, PerLTQA~\citep{perltqa}, provides an initial taxonomy but exhibits key shortcomings that limit its practical applicability: it assigns only a single label per memory, overlooking the multi-topic nature of real conversations; its semantic memory category covers only user-related content while neglecting general world knowledge; and it uniformly treats all dialogues as episodic memory, conflating the dialogue form with its content type. These limitations motivate the need for a more comprehensive and fine-grained classification benchmark grounded in real-world dialogue scenarios.

To this end, we introduce \ourdata, a memory multi-class dataset of 6,000 samples drawing on Tulving's distinction theory~\citep{cognitive_psychology}. 
TriMEM provides fine-grained type annotations: \textit{Episodic Memory} for event-specific recollections, \textit{Personal Semantic Memory} for user-related knowledge, and \textit{General Semantic Memory} for factual world knowledge, across both user-user and user-assistant scenarios. The detailed classification criteria are provided in Table~\ref{class_reason}, and the dataset statistics are presented in Figure~\ref{statistic} in the Appendix. Due to the space limit, further details about \ourdata construction and evaluation can be found in Appendix~\ref{app:dateset}.
\vspace{-0.5em}

\begin{figure*}
    \centering
\includegraphics[width=1\linewidth]{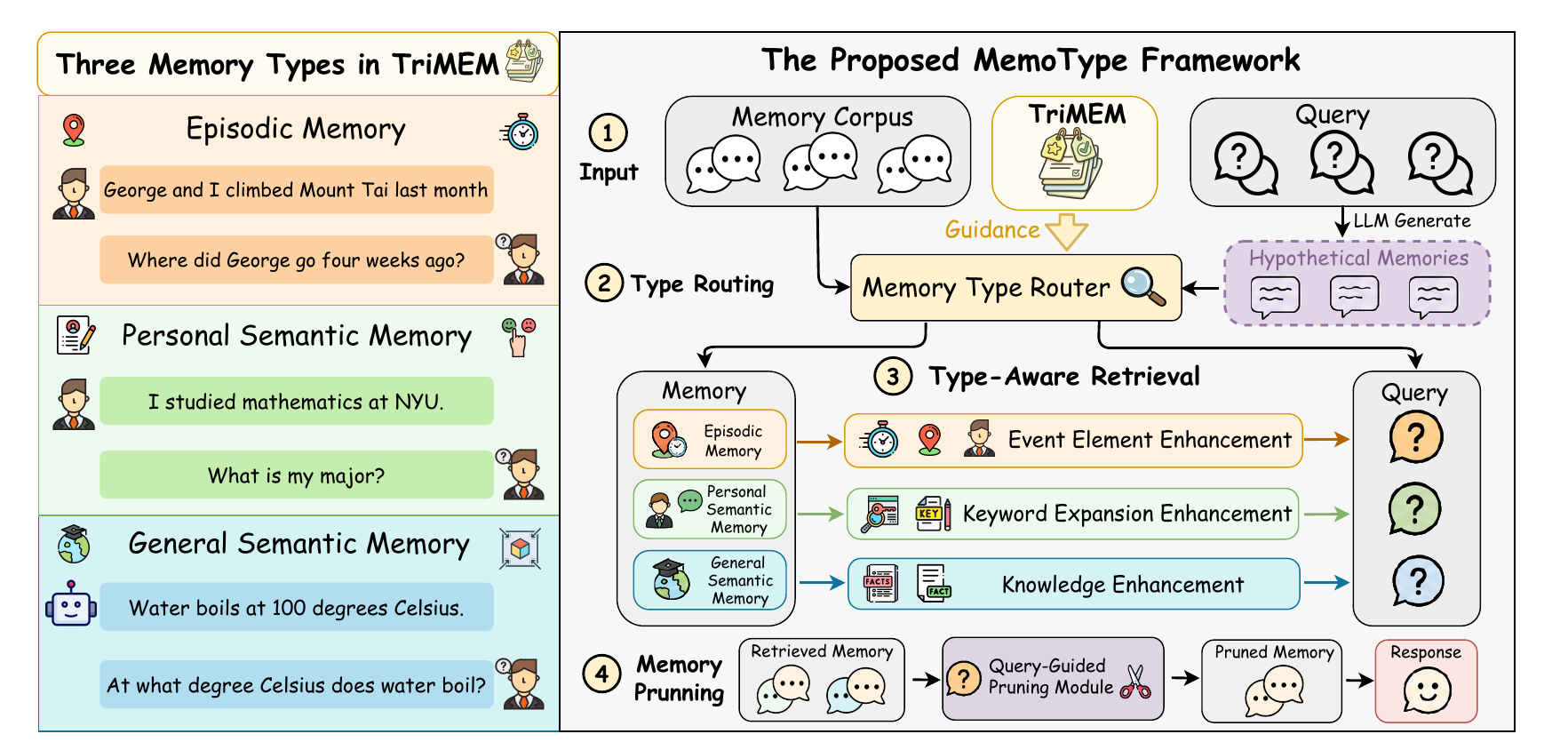}
    \caption{Three Memory Types in \ourdata (Left) and the proposed \ourmodel Framework (Right).}
    \label{framework}
    \vspace{-1em}
\end{figure*}

\subsection{Type Router}
\label{sec:router}
\vspace{-0.2em}

\textbf{Memory Type Router.} With the assistance of the \ourdata dataset, we fine-tune a lightweight Qwen3-1.7B model with LoRA as the memory type router, formulated as a multi-label classification task to accommodate the multi-topic nature of real-world dialogues. Compared to prompt-based classification with large-scale LLMs, our fine-tuned router achieves superior accuracy (Table ~\ref{trimem_class} in Appendix~\ref{app:trimem_evaluation}). Notably, as shown in Table~\ref{table:bert_ablation} in Appendix~\ref{app:router_architecture}, our experiments show that even lightweight encoders such as BERT can achieve competitive routing performance, indicating that the dataset's quality, rather than model scale, is the key enabler for accurate memory classification.
% \begin{equation}
% \mathcal{L} = - \frac{1}{N} \sum_{i=1}^{N} \sum_{c=1}^{C} \left[ y_{i,c} \cdot \log \sigma(z_{i,c}) + (1 - y_{i,c}) \cdot \log (1 - \sigma(z_{i,c})) \right],
% \end{equation}
% \begin{equation}
% \begin{split}
% \mathcal{L} = - \frac{1}{N} \sum_{i=1}^{N} \sum_{c=1}^{C} \Big[ & y_{i,c} \cdot \log \sigma(z_{i,c})  + (1 - y_{i,c}) \cdot \log (1 - \sigma(z_{i,c})) \Big],
% \end{split}
% \end{equation}
% where  $N$ is the number of samples, $C$ is the classes number. $y_{i,c}$ represents the ground truth label for sample \(i\) and class \(c\), where \(y_{i,c} \in \{0, 1\}\).  \(\sigma(\cdot)\) is the sigmoid activation function.

\textbf{Query Type Router.} 
 Following memory classification, a challenge emerges in achieving accurate query routing. Direct query classification is unreliable due to the semantic asymmetry between queries and memories. Queries are typically short and underspecified (e.g., "What is my major?"), while memories contain richer information, making it difficult to classify queries directly into types.

To bridge this gap, we propose to transform queries into the memory space. 
Specifically, given a query $q$, we generate $K$ \textit{hypothetical memories}~\citep{hyde} that the query might retrieve:
\begin{equation}
\{\tilde{m}_i\}_{i=1}^{K} = \mathcal{G}_{\text{LLM}}(q; \mathcal{P}), \quad \tilde{m}_i \in \mathcal{M},
\label{eq:hyp_gen}
\end{equation}
where $\mathcal{G}_{\text{LLM}}$ denotes the hypothetical memory generation function, 
$\mathcal{P}$ is the generation prompt, 
and $\mathcal{M}$ represents the memory space. The generated hypothetical memories $\{\tilde{m}_i\}_{i=1}^{K}$ are subsequently classified by the pre-trained memory router $f_{\phi}$, 
and the query type is determined as the most frequent category among the $K$ predictions.

\textbf{Type-Aware Memory Filtering.} 
With the memory router and query type router established, MemoType enables \textit{type-aware retrieval} by constraining the search space based on the predicted types of hypothetical memories.
Specifically, let $\mathcal{M} = \bigcup_{c \in \mathcal{C}} \mathcal{M}_c$ denote the full memory corpus partitioned by category. 
Given the $K$ hypothetical memories $\{\tilde{m}_i\}_{i=1}^{K}$ generated for query $q$, 
we first obtain their predicted types $\{c_i = f_{\phi}(\tilde{m}_i)\}_{i=1}^{K}$ and construct the candidate category set: $\mathcal{C}_q = \{c_i \mid i = 1, \ldots, K\}$. The retrieval is then performed as follows:
\begin{equation}
\small
\mathcal{R}(q) = \text{Retrieve}\left(q, \bigcup_{c \in \mathcal{C}_q} \mathcal{M}_c\right).
\label{eq:filter}
\end{equation}
This filtering strategy significantly reduces the retrieval space, eliminating approximately 50\% of unrelated memories across various datasets. We set K=5 to balance coverage and robustness.

\subsection{Type-Specific Retrieval Strategies}
\label{sec:type_specific}
As previously mentioned, most existing approaches employ a single strategy to process all memories, resulting in suboptimal performance. Therefore, in this subsection, we will adopt different strategies for different memory types to enhance the final performance.

\textbf{Episodic Memory.}
Episodic memories are inherently compositional, encoding structured elements such as temporal information, participants, and other event elements.
When encoded as monolithic units, these heterogeneous elements are compressed into a single embedding vector, causing fine-grained cues to be overshadowed by the dominant event semantics.
This observation motivates a decompose-then-align strategy: rather than matching at the holistic level, we explicitly decompose both queries and memories into typed elements and compute element-wise alignment.

Formally, let $\mathcal{N} = \{T, P, L, E\}$ denote the temporal information, participants, location, and event context, respectively.
For a query $q$, the LLM extracts elements $\mathcal{E}_q = \{e_t\}_{t \in \mathcal{N}}$.
The relevance score between query $q$ and the $i$-th memory $m^i$ is then computed as:
\begin{equation}
    s_{\text{EM}}(q, m^i) =  \text{sim}(q, m^i) +   \sum_{t \in \mathcal{N}} \text{sim}(e_t, m^i_t),
\end{equation}
where $m^i_t$ is the event element extracted from the $i$-th memory $m^i$.

We note that episodic memory, which often requires temporal reasoning, constitutes the most challenging category among the three types. While element decomposition introduces additional computational cost, the routing mechanism (\S\ref{sec:router}) restricts this strategy to the episodic subset, avoiding the corpus-wide structural expansion, which substantially amortizes the per-query overhead.

\textbf{Personal Semantic Memory.} The core challenge of PM retrieval lies in the semantic asymmetry between queries and memories.
Queries targeting personal knowledge tend to be underspecified
(e.g., ``What is your occupation?''), while the corresponding
memories contain substantially richer context.
This asymmetry makes it difficult to generate faithful hypothetical
memories from the query side.

We therefore reverse the expansion direction: instead of enriching
the query, we enrich each memory $m_i$ by generating $K$
hypothetical queries
$\{\tilde{q}_j\}_{j=1}^{K} = \mathcal{G}_{\text{LLM}}(m_i)$
that the memory could plausibly answer.
Since the memory carries richer context, the generated questions
are more faithful and diverse.
The relevance between query $q$ and memory $m_i$ is measured as:
\begin{equation}
s_{\text{PM}}(q, m_i) = \max_{j \in \{1, \ldots, K\}}
\text{sim}(q, \tilde{q}_j).
\end{equation}
The final ranking fuses $s_{\text{PM}}$ with the original
query-memory similarity via Reciprocal Rank
Fusion.

 \textbf{General Semantic Memory.} 
Unlike the preceding two types, general semantic memory primarily
encompasses factual world knowledge whose expression is relatively
standardized (e.g., ``Water boils at 100°C'').
This regularity enables LLMs to generate high-quality hypothetical
memories that closely approximate the actual memory content,
making query-side expansion effective without additional structural
decomposition or direction reversal.
We therefore adopt the hypothetical memory strategy and aggregate
the retrieval results from both the original query and the
hypothetical memory via Reciprocal Rank
Fusion~\citep{rrf}.

\subsection{Memory Pruning Module}
Retrieved memory segments are typically multi-turn dialogue
excerpts that span multiple topics within a single segment.
As a result, even when a segment is correctly retrieved,
it often contains substantial content unrelated to the query,
which may mislead the generator and degrade response quality.
This issue is more pronounced in dialogue memory than in
standard document retrieval, where passages tend to be
topically coherent.

To mitigate this, we introduce a \textit{Query-Conditioned
Memory Pruning} module that filters retrieved content before
generation. Let $\mathcal{M}_q = \{m_1, m_2, \ldots, m_n\}$
denote the retrieved memory segments for query $q$.
The pruning operation selects a relevant subset via
query-conditioned assessment:
\begin{equation}
\mathcal{M}^p_q = \bigl\{ m_i \in \mathcal{M}_q
\mid \mathcal{J}_{\text{LLM}}(m_i, q) = 1 \bigr\},
\end{equation}
where $\mathcal{J}_{\text{LLM}}(m_i, q)$ is a
binary relevance judgment, indicating
whether memory $m_i$ contains information pertinent to $q$.
The pruned memory $\mathcal{M}^p_q$ is then utilized for
downstream response generation. 

\vspace{-0.5em}

\section{Theoretical Analysis}

In this section, we provide a theoretical analysis demonstrating that single embedding retrieval methods face limitations when the corpus exhibits semantic heterogeneity.

\textbf{Problem Setup}. Consider a corpus $\mathcal{D} = \{d_1, d_2, \ldots, d_N\}$ where documents are drawn from $C$ distinct latent distributions $\mathcal{D}_1, \ldots, \mathcal{D}_C$, with categories $S_1, \ldots, S_C$ respectively. In this section, we consider a unified embedding-based retrieval strategy, where a shared encoder $\phi$ maps queries and documents to a $d$-dimensional space, and documents are ranked by inner product similarity $\langle \phi(q), \phi(d) \rangle$. Without loss of generality, we assume all embeddings $\phi(\cdot)$ are $\ell_2-$normalized, as retrieval depends only on the relative ranking. For category $S_i$, we denote its embedding centroid as $\mu_i = \mathbb{E}_{d \sim \mathcal{D}_i}[\phi(d)]$ and its intra-category standard deviation as $\sigma_i$.

To formalize corpus heterogeneity, we define the angular separation between categories $S_i$ and $S_j$ as:
$$\delta_{ij} = 1 - \langle v_i^*, v_j^* \rangle,$$
where $v_i^* = \mu_i / \|\mu_i\|$ is the direction of $S_i$'s centroid.

\begin{assumption}
\label{assump:separation}
There exists a constant $\delta > 0$ such that $\delta_{ij} \geq \delta$ satisfies for all distinct category pairs $i \neq j$.
\end{assumption}

\begin{remark}
    This assumption requires only that different categories have non-identical centroid directions. When categories are semantically distinct or form clusters, this condition is naturally satisfied. Figure~\ref{fig:longmem_s_tsne} and Figure~\ref{fig:semantic_clustering} in the Appendix validate the rationality of the assumption.
\end{remark}

\begin{restatable}[Upper Bound for Single Retrieval]{theorem}{precisionthm}
\label{thm:precision}
Let $\mathcal{D}$ satisfy Assumption~\ref{assump:separation} with angular separation $\delta > 0$. Assume documents are uniformly distributed across $C$ categories. Then, for any retrieval encoder $\phi$, the expected query-document similarity has an upper bound:
$$\mathbb{E}_{q \sim \mathcal{Q}, d \sim \mathcal{D}}[\langle e_{query}, e_{doc} \rangle] \leq 1 - \frac{(C-1)\delta}{2C},$$
where $e_{query} = \phi(q)/\|\phi(q)\|$ and $e_{doc} = \phi(d)/\|\phi(d)\|$.
\end{restatable}

\begin{remark}
    The theorem reveals that the expected query-document alignment under a shared encoder is fundamentally constrained by two factors: the number of categories $C$ and the inter-category separation $\delta$. As category diversity increases or categories become more separated, this alignment ceiling tightens. 
\end{remark}
\begin{remark}
   The bound in Theorem~3.3 implies that improvements in alignment for one category necessarily come at the cost of suppressing alignment for others. For these underserved categories, the resulting low similarity scores push their relevant documents to lower positions in the ranking over the heterogeneous corpus, thus degrading retrieval precision. Thus, \textit{\uline{an inherent limitation emerges: a single retrieval strategy cannot excel at all categories simultaneously}}, which motivates the type-aware design in Section~\ref{sec:type_specific}.
\end{remark}

\vspace{-0.5em}

% Please add the following required packages to your document preamble:
% \usepackage[normalem]{ulem}
% \useunder{\uline}{\ul}{}
\vspace{-0.5em}

\section{Experiment}
\vspace{-0.5em}

\subsection{Experimental Settings}
\textbf{Dataset.} In the section, we evaluate our model performance on four long-term memory datasets: LongMemEval-S \citep{longmemeval}, LoCoMo \citep{locomo}, LongMemEval-M \citep{longmemeval}, and PerLTQA~\citep{perltqa}. These datasets are designed to assess both retrieval and generation tasks. Detailed dataset statistics are provided in the Appendix \ref{dataset}.

\begin{table*}[t]
\fontsize{8}{9}\selectfont 
\centering
\caption{Retrieval Performance. All methods are based on Contriever as the retriever.}
\label{retrieval_performance}
\vspace{-0.5 em}
\resizebox{\linewidth}{!}{
\renewcommand\arraystretch{1}
\tabcolsep=0.07cm
\scalebox{0.8}{
\begin{tabular}{ccccccccc}
\toprule
Method     & Recall@1       & NDCG@1         & Recall@3       & NDCG@3         & Recall@5       & NDCG@5         & Recall@10       & NDCG@10         \\ \midrule
\multicolumn{9}{c}{LongMemEval-S}                                                                                                                  \\ \midrule
HyDE~\citep{hyde}      & 27.23          & 27.23          & 53.19          & 37.82          & 65.32          & 42.84          & 82.34           & 49.83          \\
Mill~\citep{mill}     & 38.72          & 38.72          & 67.66          & 50.17          & 78.94          & 56.10          & 90.64           & 61.34          \\
Query2Doc~\citep{query2doc}& 25.96          & 25.96          & 50.64          & 36.43          & 64.89          & 41.66          & 81.06           & 48.38          \\
SeCom~\citep{secom}    & 45.32          & 45.32          & 74.04          & 53.42          & 83.19          & 59.42          & 91.28           & 64.66          \\
HippoRAG2~\citep{hipporag2}& 50.64          & 50.64          & \uline{80.85}    & \uline{60.59}    & \uline{88.51}    & \uline{66.48}    & \textbf{93.53}  & \uline{71.05}    \\
A-Mem~\citep{a-mem}     & 48.20          & 48.20          & 74.80          & 56.63          & 82.80          & 61.52          & 91.40           & 66.83          \\
MemoryOS~\citep{memoryos}  & 35.70          & 35.70          & 61.80          & 38.30          & 73.07          & 43.82          & 85.39           & 48.56          \\
LightMem~\citep{lightmem}  & \uline{56.99}    & \uline{56.99}    & 73.28          & 51.59          & 79.12          & 55.57          & 88.10           & 60.24          \\
\ourmodel (Ours)      & \textbf{59.15} & \textbf{59.15} & \textbf{83.19} & \textbf{66.74} & \textbf{90.21} & \textbf{71.60} & \uline{92.77}     & \textbf{74.85} \\ \midrule
\multicolumn{9}{c}{LoCoMo}                                                                                                                         \\ \midrule
HyDE~\citep{hyde}     & 25.33          & 25.33          & 42.55          & 36.24          & 50.45          & 39.60          & 59.82           & 42.65          \\
Mill~\citep{mill}     & 19.84          & 19.84          & 35.60          & 30.16          & 42.75          & 33.10          & 53.12           & 36.57          \\
Query2Doc~\citep{query2doc}& 23.11          & 23.11          & 39.78          & 33.94          & 47.48          & 37.15          & 59.37           & 41.11          \\
SeCom~\citep{secom}    & 23.77          & 23.77          & 39.38          & 34.06          & 45.57          & 36.60          & 55.44           & 39.76          \\
HippoRAG2~\citep{hipporag2}& 24.97          & 24.97          & 39.98          & 33.84          & 45.87          & 36.28          & 53.17           & 38.64          \\
A-Mem~\citep{a-mem}     & \uline{28.80}    & \uline{28.80}    & \uline{46.12}    & \uline{39.82}    & \uline{53.32}    & \uline{42.87}    & \uline{63.29}     & \uline{46.14}    \\
MemoryOS~\citep{memoryos}  & 19.89          & 19.89          & 33.79          & 30.02          & 40.28          & 32.82          & 49.09           & 35.64          \\
LightMem~\citep{lightmem}  & 22.56          & 22.56          & 31.22          & 26.65          & 36.30          & 28.56          & 41.84           & 30.25          \\
\ourmodel (Ours)      & \textbf{31.07} & \textbf{31.07} & \textbf{47.73} & \textbf{40.97} & \textbf{54.08} & \textbf{43.66} & \textbf{65.01}  & \textbf{47.12} \\ \midrule
\multicolumn{9}{c}{LongMemEval-M}                                                                                                                  \\ \midrule
HyDE~\citep{hyde}     & 21.70          & 21.70          & 40.43          & 27.90          & 50.64          & 31.99          & 63.62           & 36.71          \\
Mill~\citep{mill}     & 28.51          & 28.51          & 52.77          & 37.48          & 63.40          & 42.73          & 76.81           & 48.04          \\
Query2Doc~\citep{query2doc}& 19.15          & 19.15          & 37.02          & 25.66          & 44.68          & 28.96          & 61.49           & 34.03          \\
SeCom~\citep{secom}    & \uline{31.91}    & \uline{31.91}    & \uline{55.11}    & \uline{38.16}    & \uline{64.47}    & \uline{42.38}    & \uline{75.96}     & \uline{47.06}    \\
MemoryOS~\citep{memoryos}  & 28.18          & 28.18          & 47.39          & 30.30          & 55.95          & 33.78          & 67.22           & 37.41          \\
LightMem~\citep{lightmem}  & 27.14          & 27.14          & 31.73          & 31.16          & 33.19          & 31.82          & 34.45           & 32.29          \\
\ourmodel (Ours)      & \textbf{48.09} & \textbf{48.09} & \textbf{70.43} & \textbf{52.20} & \textbf{78.09} & \textbf{57.09} & \textbf{84.89}  & \textbf{61.27} \\ \midrule
\multicolumn{9}{c}{PerLTQA}                                                                                                                        \\ \midrule
HyDE~\citep{hyde}     & 17.09          & 17.09          & 56.06          & 53.16          & 66.42          & 58.01          & 79.49           & 62.52          \\
Mill~\citep{mill}     & 50.18          & 50.18          & 79.18          & 76.46          & 86.22          & 79.76          & \uline{92.25}     & 81.85          \\
Query2Doc~\citep{query2doc}& 26.15          & 26.15          & 62.26          & 59.34          & 72.90          & 64.36          & 83.47           & 68.04          \\
SeCom~\citep{secom}    & 49.51          & 49.51          & 77.59          & 75.14          & 84.81          & 78.54          & 91.97           & 81.02          \\
HippoRAG2~\citep{hipporag2}& 41.90          & 41.90          & 72.16          & 69.87          & 78.05          & 72.64          & 85.94           & 75.41          \\
A-Mem~\citep{a-mem}     & 31.61          & 31.61          & 63.11          & 60.83          & 72.06          & 65.07          & 81.08           & 68.22          \\
MemoryOS~\citep{memoryos}  & \uline{57.19}    & \uline{57.19}    & \uline{79.14}    & \uline{77.18}    & \uline{84.78}    & \uline{79.81}    & 91.61           & \uline{82.18}    \\
LightMem~\citep{lightmem}  & 49.68          & 49.68          & 75.65          & 73.25          & 82.28          & 76.34          & 89.15           & 78.73          \\
\ourmodel (Ours)      & \textbf{71.92} & \textbf{71.92} & \textbf{88.05} & \textbf{86.22} & \textbf{92.78} & \textbf{88.48} & \textbf{96.83}  & \textbf{89.90} \\ \bottomrule
\end{tabular}
}}\vspace{-1.5 em}
\end{table*}

\textbf{Metrics.} We evaluate two types of tasks: retrieval tasks and generation tasks, using metrics tailored to each. For retrieval tasks, we use Recall@k and NDCG@k. For generation tasks, we employ BLEU, BERTScore, F1, Rouge1, Rouge2, RougeLSum, and GPT4Judge \citep{zheng2023judging}. GPT4Judge leverages GPT-5.2 to assess response alignment with reference answers. 

\textbf{Baselines.} We evaluate the proposed method with nine  baselines.   (1)\textit{HyDE} \citep{hyde} (2) \textit{Mill} \citep{mill} (3) \textit{Query2Doc} \citep{query2doc}  (4) \textit{Secom} \citep{secom} (5) \textit{A-Mem} \citep{a-mem} (6) \textit{HippoRAG 2} \citep{hipporag2}  (7) \textit{Raptor}~\citep{raptor}  (8) \textit{MemoryOS}~\citep{memoryos} (9) \textit{LightMem}~\citep{lightmem}.  Due to the space limits, detailed descriptions of baselines are provided in the Appendix~\ref{baseline}.

\textbf{Implement Details.} For all methods, we adopted a consistent strategy to ensure a fair comparison. In both retrieval and generation tasks, our method and the baselines utilized Contriever as the embedding model. For generation performance, we evaluated using the top 3 retrieved turns on the LongMemEval-S and LongMemEval-M datasets, while for LoCoMo, we selected the top 10 retrieved turns for performance assessment. In our experiments, we utilize 'gpt-4o-mini' as backbone for all tasks and baselines, including memory information processing and question-answer generation. To ensure fair comparisons, all baselines are implemented with uniform generation prompts. The temperature of LLMs is fixed at 0 to ensure reproducibility.

\begin{table*}[t]\fontsize{8}{8.5}\selectfont 
\centering
\caption{Question Answer Performance. All methods are based on Contriever as the retriever. }
\label{generation_performance}
\vspace{-1 em}
\resizebox{\linewidth}{!}{
\renewcommand\arraystretch{1}
\tabcolsep=0.15cm
\scalebox{1.02}{
\begin{tabular}{ccccccccc} \toprule
Method    & GPT4Judge            & F1             & BLEU           & Rouge1         & Rouge2         & RougeL         & RougeLsum      & BERTScore      \\ \midrule
\multicolumn{9}{c}{LongMemEval-S}                                                                                                                 \\ \midrule
HyDE~\citep{hyde}     & 42.60          & 9.96           & 1.44           & 10.58          & 4.67           & 9.31           & 9.57           & 83.02          \\
Mill~\citep{mill}     & 42.20          & 9.91           & 1.55           & 10.57          & 4.54           & 9.20           & 9.45           & 82.96          \\
Query2Doc~\citep{query2doc}& 43.00          & 9.80           & 1.47           & 10.44          & 4.46           & 9.21           & 9.43           & 82.97          \\
SeCom~\citep{secom}    & 44.80          & 11.01          & 1.65           & 11.66          & 5.14           & 10.36          & 10.58          & 83.26          \\
HippoRAG2~\citep{hipporag2}& 45.40          & 10.55          & 1.65           & 11.15          & 5.17           & 9.98           & 10.17          & 83.12          \\
RAPTOR~\citep{raptor}   & 32.20          & 12.08          & 1.90           & 12.73          & 5.82           & 11.25          & 11.35          & 83.50          \\
A-Mem~\citep{a-mem}     & \uline{47.00}    & 12.12          & 1.90           & 12.81          & 5.91           & 11.38          & 11.63          & 83.58          \\
MemoryOS~\citep{memoryos}  & 43.40          & 13.50          & 2.09           & 14.16          & 6.60           & 12.63          & 13.00          & 83.78          \\
LightMem~\citep{lightmem}  & 43.40          & \uline{14.45}    & \uline{2.28}     & \uline{15.14}    & \uline{6.69}     & \uline{13.59}    & \uline{13.85}    & \uline{84.01}    \\
\ourmodel (Ours)      & \textbf{56.80} & \textbf{20.13} & \textbf{4.66}  & \textbf{21.83} & \textbf{10.50} & \textbf{20.09} & \textbf{20.18} & \textbf{85.19} \\ \midrule
\multicolumn{9}{c}{LoCoMo}                                                                                                                        \\ \midrule
HyDE~\citep{hyde}     & 40.18          & 14.86          & 2.60           & 15.25          & 7.35           & 14.14          & 14.15          & 84.48          \\
Mill~\citep{mill}     & 37.41          & 14.50          & 2.49           & 14.88          & 7.13           & 13.84          & 13.83          & 84.38          \\
Query2Doc~\citep{query2doc}& 39.83          & 14.57          & 2.50           & 14.95          & 7.16           & 13.88          & 13.87          & 84.44          \\
SeCom~\citep{secom}    & 38.57          & 14.85          & 2.41           & 15.30          & 7.37           & 14.25          & 14.25          & 84.44          \\
HippoRAG2~\citep{hipporag2}& 37.81          & 14.61          & 2.56           & 15.03          & 7.33           & 14.03          & 14.04          & 84.42          \\
RAPTOR~\citep{raptor}   & 31.72          & \uline{14.55}    & \uline{2.88}     & \uline{15.09}    & \uline{7.49}     & \uline{14.18}    & \uline{14.17}    & \uline{84.48}    \\
A-Mem~\citep{a-mem}     & \uline{39.98}    & 14.44          & 2.32           & 14.82          & 6.85           & 13.73          & 13.75          & 84.39          \\
MemoryOS~\citep{memoryos}  & 36.81          & 13.33          & 2.12           & 13.81          & 6.37           & 12.85          & 12.84          & 84.20          \\
LightMem~\citep{lightmem}  & 25.03          & 9.20           & 1.20           & 9.70           & 3.70           & 8.85           & 8.85           & 83.30          \\
\ourmodel (Ours)      & \textbf{48.79} & \textbf{22.22} & \textbf{4.66}  & \textbf{23.24} & \textbf{11.88} & \textbf{21.84} & \textbf{21.80} & \textbf{85.82} \\ \midrule
\multicolumn{9}{c}{LongMemEval-M}                                                                                                                 \\ \midrule
HyDE~\citep{hyde}     & 34.60          & 9.01           & 1.26           & 9.71           & 3.99           & 8.46           & 8.67           & 82.70          \\
Mill~\citep{mill}     & 35.40          & 8.79           & 1.29           & 9.47           & 3.89           & 8.14           & 8.39           & 82.58          \\
Query2Doc~\citep{query2doc}& 34.80          & 8.66           & 1.20           & 9.34           & 3.88           & 7.99           & 8.23           & 82.61          \\
SeCom~\citep{secom}    & 34.40          & 10.13          & 1.37           & 10.83          & 4.65           & 9.51           & 9.69           & 83.07          \\
MemoryOS~\citep{memoryos}  & \uline{36.20}    & \uline{12.22}    & \uline{1.79}     & \uline{12.87}    & \uline{5.87}     & \uline{11.49}    & \uline{11.79}    & \uline{83.57}    \\
LightMem~\citep{lightmem}  & 25.20          & 9.30           & 1.11           & 9.92           & 4.15           & 8.65           & 8.92           & 82.43          \\
\ourmodel (Ours)      & \textbf{50.20} & \textbf{18.77} & \textbf{4.24}  & \textbf{20.36} & \textbf{9.19}  & \textbf{18.56} & \textbf{18.64} & \textbf{85.02} \\ \midrule
\multicolumn{9}{c}{PerLTQA}                                                                                                                       \\ \midrule
HyDE~\citep{hyde}     & 50.56          & 33.96          & 8.48           & 35.61          & 19.55          & 29.69          & 29.74          & 89.47          \\
Mill~\citep{mill}     & 60.11          & 38.55          & 11.21          & 40.17          & 23.44          & 34.00          & 34.05          & 90.26          \\
Query2Doc~\citep{query2doc} & 52.85          & 35.01          & 9.14           & 36.64          & 20.42          & 30.65          & 30.74          & 89.66          \\
HippoRAG2~\citep{hipporag2}& 56.73          & 37.04          & 10.11          & 38.71          & 22.17          & 32.75          & 32.86          & 89.97          \\
RAPTOR~\citep{raptor}    & 40.84          & 32.51          & 8.56           & 34.38          & 18.3           & 28.65          & 28.72          & 89.13          \\
A-Mem~\citep{a-mem}     & 52.01          & 35.13          & 9.24           & 36.76          & 20.58          & 30.86          & 30.93          & 89.62          \\
SeCom~\citep{secom}    & 57.68          & \uline{37.49}    & \uline{10.27}    & \uline{39.18}    & \uline{21.91}    & \uline{32.85}    & \uline{32.92}    & \uline{90.09}    \\
MemoryOS~\citep{memoryos}  & \uline{60.08}    & 36.32          & 9.89           & 37.95          & 21.46          & 31.73          & 31.79          & 89.93          \\
LightMem~\citep{lightmem}  & 31.36          & 32.13          & 7.69           & 34.12          & 16.66          & 28.01          & 28.06          & 89.15          \\
\ourmodel (Ours)      & \textbf{62.83} & \textbf{47.26} & \textbf{20.65} & \textbf{51.20} & \textbf{33.02} & \textbf{44.78} & \textbf{44.79} & \textbf{91.61}   \\ \bottomrule      
\end{tabular}}} 
\vspace{-1em}
\end{table*}

\vspace{-0.5em}

\subsection{Overall Results}
\vspace{-0.5em}

In this paper, we evaluate two types of tasks, retrieval and generation tasks, with their respective results presented in Table \ref{retrieval_performance} and Table \ref{generation_performance}. Notably, Raptor involves generating a new text process, making it impossible to assess their retrieval performance. Additionally, for A-Mem, HippoRAG2, and Raptor, the memory construction latency on LongMemEval-M exceeds one week, and thus, we do not report their performance in this study.

\textbf{Retrieval Results.} As shown in Table~\ref{retrieval_performance}, MemoType consistently achieves the best performance across all four datasets. On LongMemEval-S, MemoType achieves a Recall@1 of 59.15\%, outperforming the strongest baseline LightMem by 2.16\% and the structural method HippoRAG2 by 8.51\%. On the more challenging LongMemEval-M dataset, MemoType achieves a Recall@1 of 48.09\%, surpassing SeCom by over 16 points, demonstrating the substantial advantage of type-aware retrieval over uniform segmentation-based approaches. On LoCoMo, MemoType outperforms all baselines including A-Mem by 2.27\% in Recall@1, confirming its effectiveness on user-user dialogue scenarios. On PerLTQA, MemoType achieves 71.92\% Recall@1, exceeding the second-best method MemoryOS by 14.73\%. These results validate that leveraging memory type diversity with tailored strategies yields consistent improvements over both query expansion and structural augmentation methods across diverse memory benchmarks.

\textbf{Question Answer Results.} To validate the question answer performance of the proposed method, we conduct the QA experiment in Table \ref{generation_performance}. MemoType consistently achieves the best performance across all four datasets. On LongMemEval-S, MemoType attains a GPT4Judge score of 56.80 and a BLEU score of 4.66, surpassing the strongest baseline LightMem by over 13 points and 2.38 points respectively, which highlights the limitations of unified retrieval strategies in generating high-quality answers. On LoCoMo, MemoType outperforms all baselines by a notable margin, achieving an F1 of 22.22 and a GPT4Judge of 48.79, confirming its effectiveness in leveraging memory diversity for user-user dialogue scenarios. These results underscore the robustness and generalizability of our method across diverse memory benchmarks.

\begin{table*}[t]\fontsize{8}{8.5}\selectfont 
\centering
\caption{The ablation study of the proposed type-aware strategy with single retrieval strategy on three datasets.}
\label{tab:ablation_strategy}
\vspace{-1 em}
\resizebox{\linewidth}{!}{
\renewcommand\arraystretch{1}
\tabcolsep=0.1cm
\scalebox{1.02}{ \begin{tabular}{cccccccccc} \toprule
\multirow{2}{*}{Strategy} & \multicolumn{3}{c}{LongMemEval-S}                & \multicolumn{3}{c}{LoCoMo}                       & \multicolumn{3}{c}{PerLTQA}                      \\ \cmidrule(lr){2-4} \cmidrule(lr){5-7} \cmidrule(lr){8-10}
                          & Recall@1       & Recall@3       & Recall@5       & Recall@1       & Recall@3       & Recall@5       & Recall@1       & Recall@3       & Recall@5       \\ \midrule
Key Expansion             & 51.77          & 76.62          & 85.59          & 25.93          & 43.25          & 51.06          & 66.10          & 83.64          & 88.05          \\
Hypothetical Query        & 53.91          & 78.29          & 87.06          & 24.42          & 43.00          & 50.10          & 62.65          & 82.90          & 89.08          \\
Hypothetical Memory       & 52.03          & 77.45          & 86.22          & 25.23          & 42.30          & 51.11          & 66.24          & 83.88          & 88.34          \\ 
 Type-Aware Strategy   & \textbf{59.15} & \textbf{83.19} & \textbf{90.21} & \textbf{31.07} & \textbf{47.73} & \textbf{54.08} & \textbf{71.92} & \textbf{88.05} & \textbf{92.78} \\ \cmidrule{1-10}
Improvement               & 5.24           & 4.90           & 3.15           & 5.14           & 4.48           & 2.97           & 5.68           & 4.17           & 3.70     \\ \bottomrule
\end{tabular}
}}
\vspace{-1em}
\end{table*}

\subsection{Ablation Study}
In this section, we conduct ablation studies to analyze the effectiveness of our proposed method. Due to space limits, \uline{more experiments about router performance, cost analysis, retrieval strategy ablation, and other ablation studies can be found in Appendix~\ref{app:experiment}}.

\textbf{Type-Aware Retrieval Strategy.} To validate the type-aware retrieval strategy, we compare it with three single-strategy baselines where all memory types employ the same retrieval enhancement, as shown in Table~\ref{tab:ablation_strategy}. Our type-aware strategy consistently outperforms all single strategies across three datasets. On LongMemEval-S, our strategy achieves a Recall@1 of 59.15, surpassing the best single strategy Hypothetical Query by 5.24 points. Similar improvements are observed on LoCoMo and PerLTQA, with gains of 5.14 and 5.68 in Recall@1 respectively. These results confirm that assigning tailored retrieval strategies to different memory types effectively leverages their unique characteristics and yields substantial performance gains over any uniform approach.

\begin{wrapfigure}{r}{0.47\textwidth}
\vspace{-3pt}
\centering
\setlength{\abovecaptionskip}{0pt}
\includegraphics[width=1\linewidth]{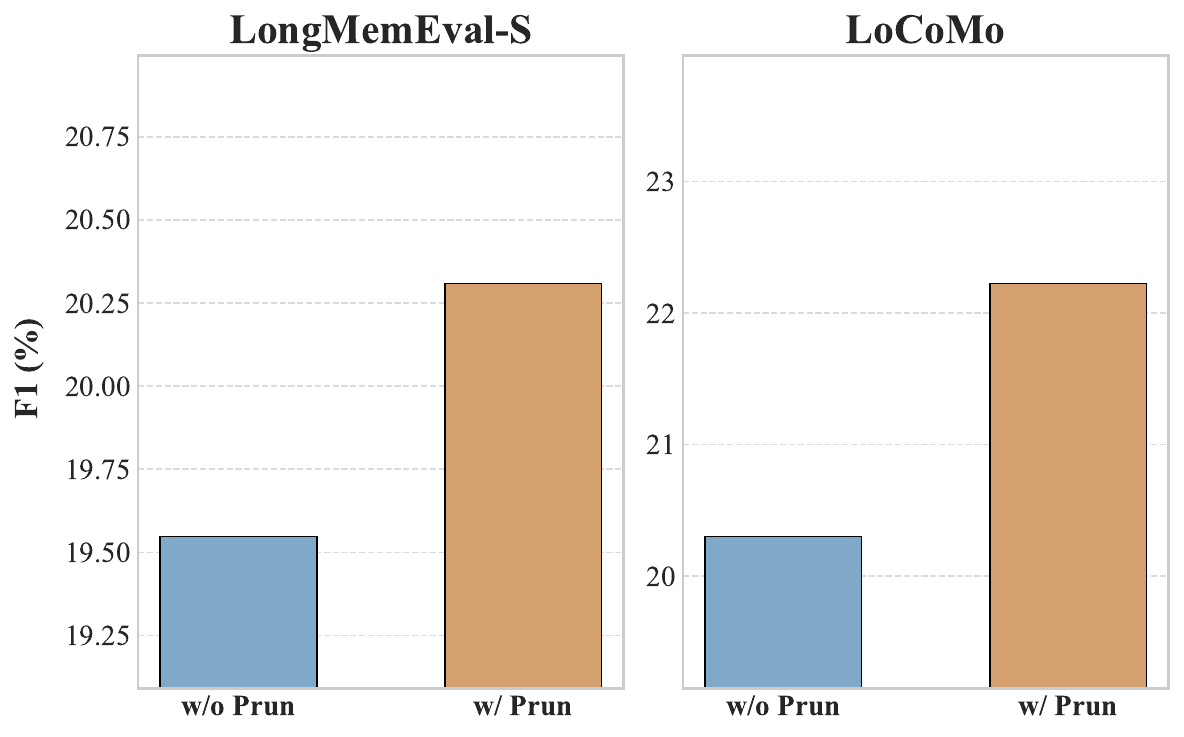}
\caption{The ablation study of the memory pruning module on two benchmarks.}
\label{fig:pruning}
\vspace{-8pt}
\end{wrapfigure}

\textbf{Memory Pruning Module.} To evaluate the effectiveness of the proposed query-conditioned memory pruning module, we compare the generation performance with and without pruning, as shown in Figure~\ref{fig:pruning}. On LongMemEval-S, applying the pruning module improves F1 from 19.55 to 20.31. These improvements demonstrate that retrieved memory segments inevitably contain query-irrelevant information that introduces noise into the generation process. By adaptively filtering out such irrelevant content before feeding memories into the LLM, the pruning module effectively reduces noise and enhances the quality of generated responses across different datasets. Due to the space limits, more experimental results about pruning module ablation can be found in Appendix ~\ref{app:pruning_ablation}.

\vspace{-1em}

\section{Related Work}

\textbf{RAG-based Memory technology}
RAG-based memory strategies can be broadly classified into unstructured and structured types. Unstructured RAG \citep{memochat, mms, cfgm, memorybank, mem0} stores history as dense chunks or LLM summaries and retrieves by similarity; representative systems include SeCom’s segment-level denoising \citep{secom}. Structured RAG \citep{hipporag, graphrag, g-retriever, graph-cot, lightrag} introduces relations via graphs or trees: Zep’s \citep{zep} Graphiti integrates conversations and business data into a temporal KG, MemTree \citep{memtree} organizes a dynamic tree, and H‑MEM \citep{h-mem}. Structural information augmentation can improve performance on complex reasoning tasks, but it often incurs inevitable computational overhead. For simpler problems, such enhancements may introduce noise and degrade performance. Our proposed Memotype addresses this issue effectively by adopting tailored strategies for each.

\textbf{Agent Memory}
Beyond RAG, agents learn or structure memory with richer mechanisms \citep{memoryllm, m+, mirix, memoryos, survey-memory}. MemoryLLM \citep{memoryllm} updates parameters with a latent memory pool. Memory‑R1 \citep{memory-r1} trains a memory manager agent with RL to add/update/delete entries. Learn‑to‑Memorize models memory cycles with MoE‑gated retrieval, aggregation, and reflection; Mem$^p$~\citep{memp} distills procedural traces into lifelong skills. Through methods such as reinforcement learning, these approaches often achieve superior performance compared to traditional RAG systems. However, they typically incur higher memory construction latency. Balancing performance and latency remains an area worthy of further exploration for these methodologies.

\section{Conclusion}
In this paper, we propose \ourdata, a multi-class memory benchmark inspired by cognitive psychology, categorizing memory into \textit{Episodic Memory}, \textit{Personal Semantic Memory}, and \textit{General Semantic Memory}. Built on this foundation, we introduce \ourmodel, a novel memory framework that adaptively identifies memory types and designs tailored storage and retrieval strategies, significantly enhancing performance. We further provide a theoretical analysis showing that single-strategy retrieval faces fundamental precision bounds in heterogeneous corpora.  Extensive experiments demonstrate its effectiveness.  Whether different memory types at the corpus levels can benefit from better strategies remains a promising direction for future research.

\bibliography{neurips_2026}
\bibliographystyle{plainnat}

%%%%%%%%%%%%%%%%%%%%%%%%%%%%%%%%%%%%%%%%%%%%%%%%%%%%%%%%%%%%

\appendix
\newpage
\appendix
\onecolumn
\begin{center}
    \LARGE \textbf{Technical Appendix}
    \vspace{1em}
\end{center}
\addcontentsline{toc}{part}{Appendix}  %

\begingroup
  \etocsetnexttocdepth{subsection}     %
  \etocsettocstyle{\section*{Table of Contents}}{} %
  \localtableofcontents                %
\endgroup
\clearpage

\section{Proof of Theorem \ref{thm:precision}}
\precisionthm*
\begin{proof}
Let $u = \mathbb{E}_{q \sim \mathcal{Q}}[e_{query}]$ and $\mu_i = \mathbb{E}_{d \in S_i}[e_{doc}]$ for each category $S_i$. Since $\phi(\cdot)$ are $\ell_2-$ normalized, we have $\mathbb{E}_{d \in S_i}[e_{doc}] = \mu_i = \|\mu_i\|v_i^*$ for each category, where $\|\mu_i\| \leq 1$ by Jensen's inequality.  Besides, with queries and documents sampled independently, we have:
$$\mathbb{E}_{q \sim \mathcal{Q}, d \sim \mathcal{D}}[\langle e_{query}, e_{doc} \rangle] = \langle \mathbb{E}_{q \sim \mathcal{Q}}[e_{query}], \mathbb{E}_{d \sim \mathcal{D}}[e_{doc}] \rangle = \langle u, \mathbb{E}_{d \sim \mathcal{D}}[e_{doc}] \rangle.$$
Since documents are uniformly distributed across $C$ categories:
$$\mathbb{E}_{d \sim \mathcal{D}}[e_{doc}] = \frac{1}{C}\sum_{i=1}^{C} \mathbb{E}_{ d_i \sim \mathcal{D}_i}[e_{d}^i] = \frac{1}{C}\sum_{i=1}^{C} \|\mu_i\|v_i^*.$$
Thus:
$$\mathbb{E}_{q \sim \mathcal{Q}, d \sim \mathcal{D}}[\langle e_{query}, e_{doc} \rangle] = \frac{1}{C}\sum_{i=1}^{C} \|\mu_i\|\langle u, v_i^* \rangle = \langle u, \bar{v} \rangle,$$

where $\bar{v} = \frac{1}{C} \sum_{i=1}^{C}\|\mu_i\| v_i^*$.
By Cauchy-Schwarz inequality and $\|u\| \leq 1$:
$$\langle u, \bar{v} \rangle \leq \|u\| \|\bar{v}\| \leq \|\bar{v}\|.$$
We now bound $\|\bar{v}\|^2$:
$$\|\bar{v}\|^2 = \frac{1}{C^2}\sum_{i,j} \|\mu_i\| \|\mu_j\|\langle v_i^*, v_j^* \rangle = \frac{1}{C^2}\left[\sum_i\|\mu_i\|^2 + \sum_{i \neq j} \|\mu_i\| \|\mu_j\|\langle v_i^*, v_j^* \rangle\right].$$
Since $\|\mu_i\| \leq 1$ for all $i$. By Assumption~\ref{assump:separation}, $\langle v_i^*, v_j^* \rangle \leq 1 - \delta$ for $i \neq j$:
$$\|\bar{v}\|^2 \leq \frac{1}{C^2}[C + C(C-1)(1-\delta)] = 1 - \frac{(C-1)\delta}{C}.$$
Using $\sqrt{1-x} \leq 1 - x/2$ for $x \in [0,1]$:
$$\|\bar{v}\| \leq 1 - \frac{(C-1)\delta}{2C}.$$

Therefore, 
$$\mathbb{E}_{q \sim \mathcal{Q}, d \sim \mathcal{D}}[\langle e_{query}, e_{doc} \rangle] \leq \|\bar{v}\| \leq 1 - \frac{(C-1)\delta}{2C}.$$
\end{proof}

Since higher alignment $\langle u, v_i^* \rangle$ leads to higher retrieval scores for category $S_i$ documents, this bound implies a fundamental limit on retrieval performance: no single retrieval direction can simultaneously achieve high alignment with all category directions.

\section{TriMEM Dataset Details}
\label{app:dateset}

\subsection{Classification Rationale}

{In this paper, we will focus on text classification in the form of dialogue. Because dialogues are crucial retrieval objects in existing memory scenarios. For instance, in the two common memory benchmarks LoCoMo~\citep{locomo} and LongMemEval~\citep{longmemeval}, dialogues are the primary and exclusive retrieval objects. Moreover, in real-world scenarios, accessible data typically exists in dialogue form. We propose three memory types: Episodic Memory, Personal Semantic Memory, and General Semantic Memory. This classification is rooted in cognitive psychology research on declarative memory~\citep{cognitive_psychology}, which encompasses memory for facts and events that can be consciously recalled and linguistically described. In this subsection, we will explain the basis for our classification from the perspectives of completeness and necessity. Specifically, we will demonstrate that the three categories of memory encompass all conversational memory and justify the necessity of this classification.}

{\textbf{Completeness.} Our classification is based on well-established research in cognitive psychology~\citep{cognitive_psychology} on declarative memory. Declarative memory inherently includes both semantic memory and episodic memory, which collectively cover all forms of dialogue memory that can be expressed through language. Nondeclarative memory, such as procedural memory or priming memory, cannot be articulated linguistically and thus lies beyond the scope of this framework. Therefore, the three proposed categories comprehensively encompass all dialogue memory types without requiring additional classifications.}

{\textbf{Necessity.} The necessity of categorizing semantic memory into Personal Semantic Memory and General Semantic Memory is grounded in two critical considerations. First, the division addresses the category imbalance problem, as semantic memory accounts for a disproportionately large proportion (87.64\%) of dialogue memory compared to episodic memory (12.35\%). Subdividing semantic memory ensures a more balanced classification, mitigating the dominance of one category during model optimization and evaluation. Second, the semantic gap between Personal and General Semantic Memory is empirically validated through TSNE visualization and word cloud analysis. As shown in the Figure \ref{fig:longmem_s_tsne}, both categories exhibit distinct clustering patterns, indicating significant semantic differences between personal and general semantic memory. Additionally, we supplemented PM Word Cloud Visualizations and GM Word Cloud Visualizations as shown in Figure \ref{fig:pm_wordcloud} and Figure \ref{fig:gm_wordcloud}, revealing that General Semantic Memory emphasizes common terms like "assistant" and "offer," while Personal Semantic Memory focuses on user-specific terminology such as "planning "and "user." This further highlights the distinctions between the two memory types. This clear semantic distinction underscores the necessity of separating semantic memory into these two subcategories.}

\begin{table*}[t]
\fontsize{8}{9.7}\selectfont 
\centering
\caption{The Memory Classification Criterion in the TriMEM Benchmark. ``EM'' denotes Episodic Memory, ``PM'' denotes the Personal Semantic Memory, ``GM'' denotes the General Semantic Memory.}
\label{class_reason}
\vspace{-1 em}
\resizebox{\linewidth}{!}{
\renewcommand\arraystretch{1}
\tabcolsep=0.1cm
\scalebox{1}{
\begin{tabular}{ccc}
\toprule
Memory   Type                                 & Content Type                         & Example                                                                                          \\ \midrule
\multirow{2}{*}{EM}              & Event Occurred                       & ``Sarah and John had dinner   at a sushi restaurant downtown.'' \\
          & Events   Planned to Occur            & ``Next Month, my colleagues   are planning to see a beautiful sunset.''                            \\ \midrule
\multirow{3}{*}{PM} & User Information                     & ``Jack's degree is in   Mathematics. ''                                                           \\
  & User   Preferences and Habits        & ``I love Italian food,   especially pasta and pizza.''                                             \\
  & User   Intent                        & ``I want to buy a pair of   running shoes''                                                        \\ \midrule
\multirow{3}{*}{GM}      & General Knowledge                    & ``Water boils at 100   degrees Celsius at sea level.''                                             \\
& Broadly   Applicable Suggestions     & ``Getting up early is a good habit for maintaining good health.''                      \\
& Objective  Descriptions of Entities & ``The desks in the school   are made of wood.''                                                    \\ \midrule
\multirow{2}{*}{None}                     & Question without Information                     & ``Do you have any tips on how to  organize a shoe rack?''                                           \\
              & Sentence without Information & ``Good Morning, Peter''; ``See you next  time'' \\ \bottomrule
\end{tabular}}}
\end{table*}

\subsection{Memory Type Definitions}
\label{app:classification}
In this subsection, we offer a more comprehensive explanation of the three types of memory, as shown in Table~\ref{class_reason}.

\begin{itemize}[leftmargin=*]
    \item \textbf{Episodic Memory.} Episodic memory refers to memories of specific events or occurrences tied to contextual details such as time, location, participants, and activities. These memories can include both past experiences and plans for the future. For example: "Last summer, I attended a three-day conference in Paris where I presented my research on renewable energy." "Next Friday, I’m scheduled to meet my supervisor at 10 a.m. to review the final draft of my thesis."
    \item \textbf{Personal Semantic Memory.} Personal Semantic Memory refers to long-term knowledge related to an individual's identity, preferences, habits, or intentions. Unlike Episodic Memory, it focuses on general personal information rather than specific events. For instance: "I have been working as a data scientist for a long time, and I specialize in natural language processing." "I usually prefer hiking in the mountains over beach vacations because I enjoy the tranquility of nature."
    \item \textbf{General Semantic Memory.} General Semantic Memory encompasses universally applicable knowledge, objective descriptions, or facts about the world. It is not tied to any specific individual or event. Examples include: "The Great Wall of China is over 13,000 miles long and was constructed to protect against invasions." "Photosynthesis is the process by which green plants convert sunlight into chemical energy stored in glucose."
\end{itemize}

\subsection{Dataset Construction Pipeline}
\label{app:construction}
TriMEM is a carefully curated dataset comprising 6,000 dialogue samples, designed to evaluate long-term memory capabilities in dialogue models. The dataset is constructed with diversity and quality in mind, and its creation involves multiple steps including sampling, rewriting, and labeling, all of which are detailed below.

\textbf{Data Sampling.} TriMEM is composed of dialogue samples sourced from multiple datasets and manually constructed dialogues to ensure a balanced and realistic testbed. \textit{LongMemEval-S}: 1,000 dialogue pairs were sampled from this dataset, which contains approximately 20,000 dialogue records. To prevent over-representation, the sampling proportion was kept low at only 10\%.
\textit{LoCoMo}: 750 dialogue pairs were sampled to add variety.
\textit{PersonaMem}: 750 dialogue pairs were also sampled to further diversify the dataset.
\textit{Manually-constructed dialogues}: 500 dialogue pairs were created by human annotators to introduce additional variability and ensure coverage of edge cases not captured in the sampled datasets.

\textbf{Benchmark Selection.} {We aimed to annotate memory categories in real-world dialogue scenarios rather than generating memory-related dialogues from predefined categories, which would risk introducing bias by over-focusing on single categories. We chose  LoCoMo \citep{locomo} and  LongMemEval \citep{longmemeval} as they are widely recognized memory datasets frequently referenced in works such as A-Mem~\citep{a-mem}, mem0~\citep{mem0}, MIRIX~\citep{mirix}, and Secom~\citep{secom}. This highlights their significance and credibility. LongMemEval emphasizes user-assistant interactions, while LoCoMo focuses on user-user dialogues, aligning with our goal of covering diverse memory scenarios and topics. To further enhance the generalizability of our annotated dataset beyond evaluation-specific categories, we included  PersonaMem \citep{personamem}, a personalized memory benchmark that introduces varied user profiles. This addition strengthens the dataset’s applicability across broader dialogue classification tasks.}

\textbf{Rewriting to Prevent Overfitting.} Since both LongMemEval-S and LoCoMo are also used in the evaluation datasets of this study, special care was taken to avoid the risk of the model directly memorizing specific dialogue samples. Such memorization could compromise the evaluation of classification performance.

To address this, the sampled dialogues from these datasets were rewritten using transfer prompt, which is shown in Figure \ref{transfer1} and Figuer \ref{transfer2}. The rewriting process involved modifying the necessary entity names, rephrasing relevant topic statements, and substituting key details while preserving the original semantic structure. This ensured that the dialogues remained contextually meaningful but sufficiently distinct from their original forms, mitigating the risks of data leakage and overfitting.

\textbf{Labeling Process.}
The labeling of dialogue samples was conducted through a combination of automated and manual processes to ensure high accuracy and consistency.
 The GPT-5.2-2025-12-11 model was used to generate initial labels for the dialogue samples. These prompts were meticulously crafted and validated on a subset of the data to ensure their effectiveness for the classification task. Notably, we opted for separate prompts for each category, as this approach achieved higher classification accuracy compared to using specific prompts for all categories (Episodic Memory Classification Prompt in Figure \ref{em_prompt}, Personal Semantic Memory Classification Prompt in Figure \ref{pm_prompt},  and General Semantic Memory Classification Prompt in Figure \ref{gm_prompt}). These prompts were tailored to the unique characteristics of each memory category, ensuring precise and context-aware labeling.

\textbf{Human Verification}
\label{sec:human_verification}
To ensure the quality and reliability of the TriMEM dataset, we conducted a rigorous human verification process. Three PhD-level annotators with expertise in natural language processing and dialogue systems were recruited to independently review the dataset annotations.
\textit{(i) Verification Protocol.} A stratified random sample of 1,200 dialogue pairs (20\% of the total dataset) was selected for manual review, covering all source datasets proportionally: 400 samples from LongMemEval-S, 300 from LoCoMo, 300 from PersonaMem, and 200 from manually-constructed dialogues. Each sample was independently evaluated by at least two annotators, and disagreements were resolved through discussion until consensus was reached.
\textit{(ii) Evaluation Criteria.} Annotators assessed each sample along three dimensions: 1) Category Correctness: whether the assigned memory category accurately reflects the dialogue content; 2) Rewriting Quality: whether the rewritten dialogues preserve the original semantic structure while being sufficiently distinct from the source; and 3) Label Consistency: whether the annotations are internally consistent and free of contradictions.
\textit{(iii) Results.} The inter-annotator agreement, measured by Fleiss' $\kappa$, was 0.87, indicating near-perfect agreement. The overall annotation accuracy was 94.6\%, with category correctness at 95.2\%, rewriting quality at 93.8\%, and label consistency at 94.9\%. For the 64 samples (5.4\%) identified as erroneous or ambiguous, annotators provided corrected labels, which were incorporated into the final dataset. A second round of verification on the corrected samples confirmed 100\% agreement among all three annotators.
\textit{iv: Quality Assurance.} Beyond the sampled verification, annotators also flagged 23 additional edge cases during the review process that required further deliberation. These cases were discussed in a joint session and resolved with detailed annotation guidelines, which were subsequently applied to re-examine all similar instances in the full dataset. This iterative process further strengthened the overall quality and consistency of TriMEM.

\subsection{Dataset Statistics}
\label{app:dataset}
\vspace{-0.5 em}
\begin{wrapfigure}{r}{0.5\textwidth}
\centering
    \includegraphics[width=1\linewidth]{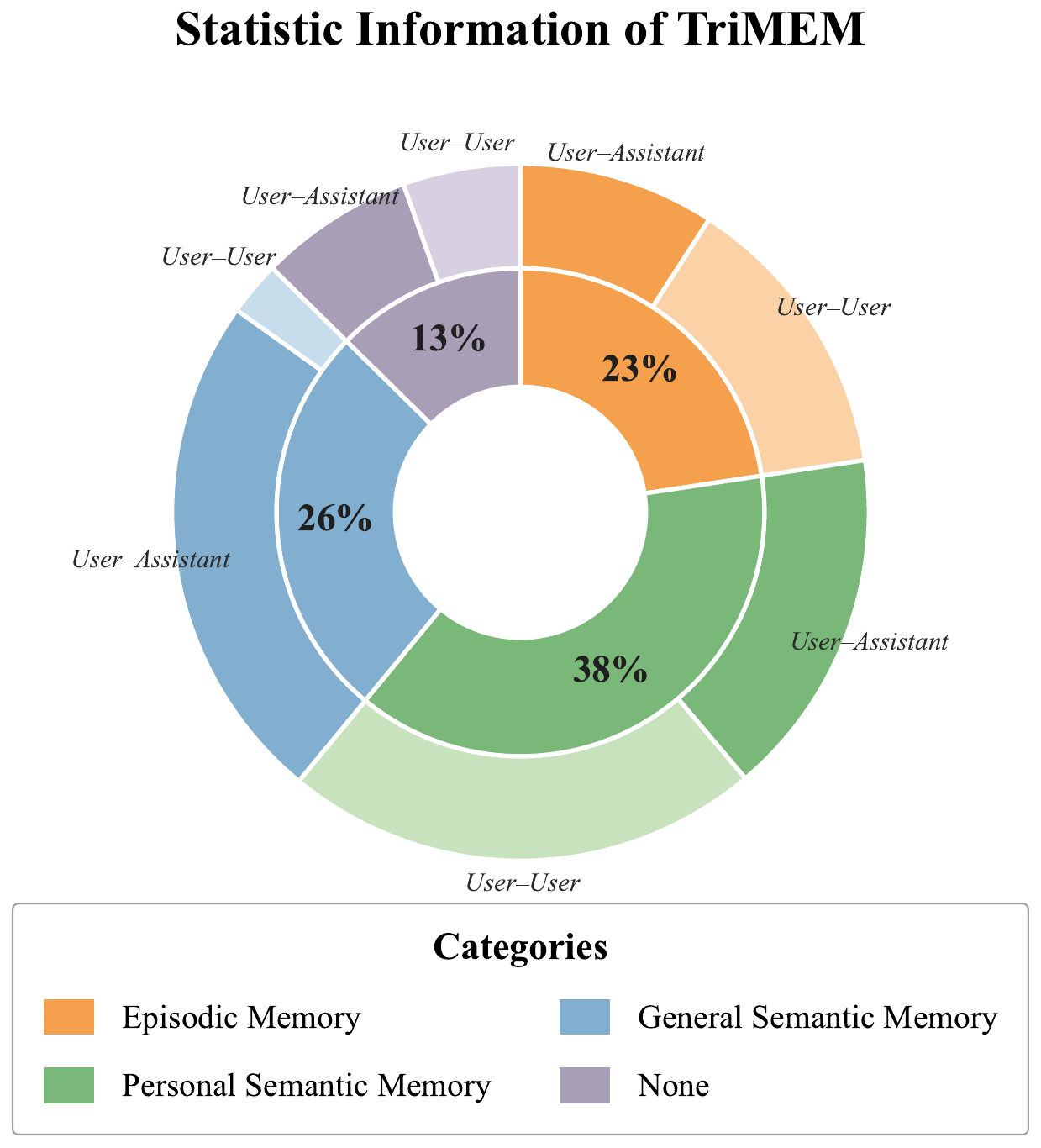}
    
    \caption[The Statistical Information of TriMEM. ``U-U'' denotes the user-to-user conversation, ``U-A'' denotes the user-to-assistant conversation.]%
          {The Statistical Information of TriMEM\protect\footnotemark.}
    \label{statistic}
    \vspace{-0.5 em}
\end{wrapfigure}
        \footnotetext{The percentage of the memory types is a relative ratio, as there is overlap between different memory types.}
The \ourdata dataset comprises 6,000 dialogues. User-assistant interactions
account for 58.77\% of the data and user-user interactions for 41.23%,
yielding a roughly 60/40 split. Each text contains an average of 91.06
tokens, indicating moderate linguistic complexity.

Coverage across the four memory categories is well distributed. Episodic
Memory covers 29.08\% of the dataset, Personal Semantic Memory 49.35\%,
General Semantic Memory 33.95\%, and None 16.25\%. Since a single dialogue
may carry multiple labels, these proportions sum above 100\%.

The breakdown by interaction type reveals clear patterns. General
Semantic Memory is dominated by user-assistant dialogues at 90.48\%,
reflecting the assistant's role in providing factual and world knowledge.
In contrast, Episodic Memory and Personal Semantic Memory slightly favor
user-user dialogues, with respective shares of 59.71\% and 57.72\%,
consistent with the more event-oriented and personal nature of peer
exchanges. The None class is approximately balanced, with user-assistant
turns at 56.82\% and user-user turns at 43.18\%.

This composition reflects the diversity of TriMEM and its capacity to
model distinct memory types across varied conversational settings. By
covering both assistant-guided and peer-to-peer exchanges, TriMEM
provides a comprehensive resource for memory-aware dialogue research
and supports tasks such as memory reasoning and personalized
conversation.

\begin{figure}[t]
\centering
\begin{minipage}{0.48\textwidth}
\centering
\includegraphics[width=1\textwidth]{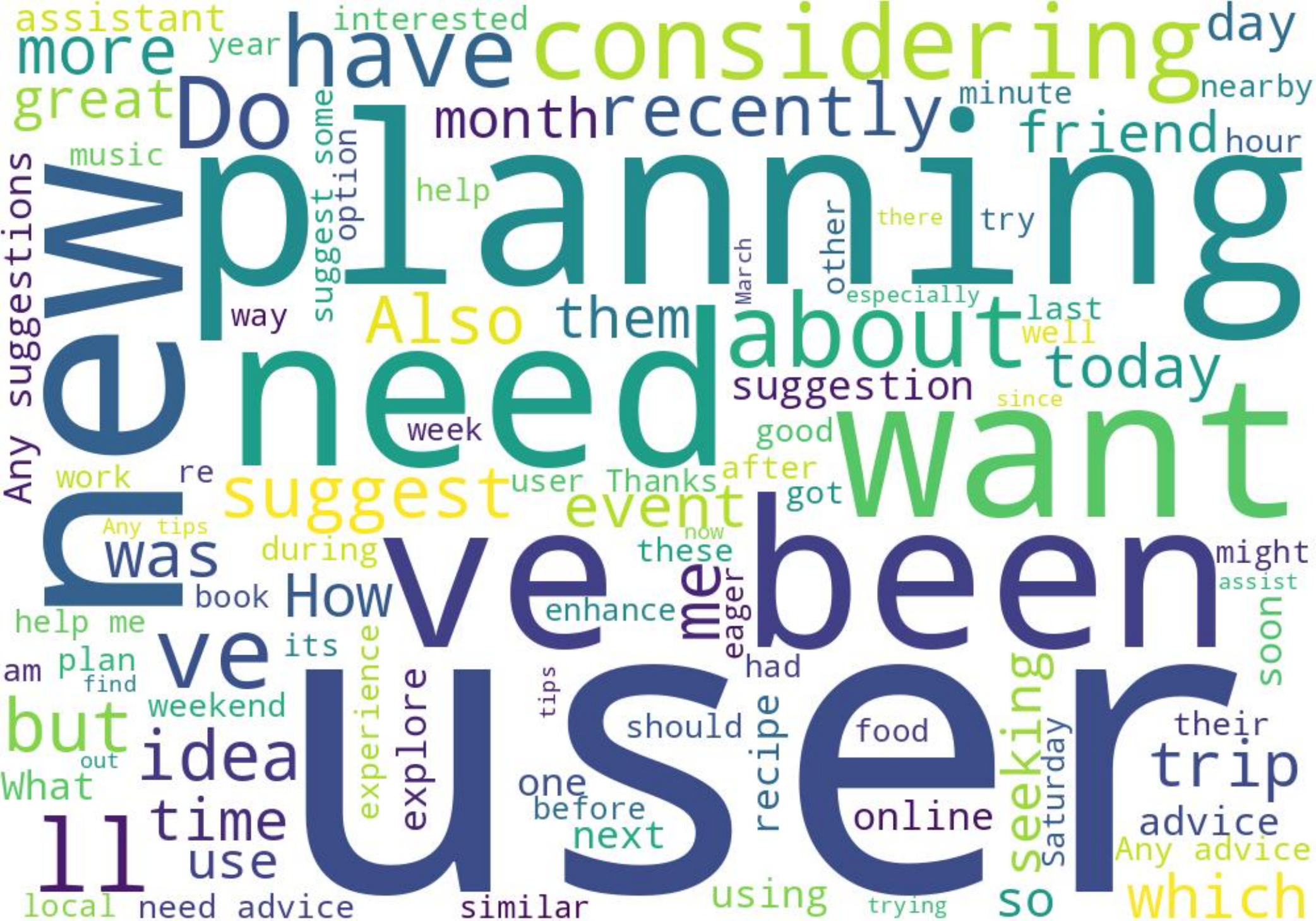}
\caption{{The Word Cloud Figure of Personal Semantic Memory on LongMemEval-S Dataset.}}
\label{fig:pm_wordcloud}
\end{minipage}
\hfill
\begin{minipage}{0.48\textwidth}
\centering
\includegraphics[width=1\textwidth]{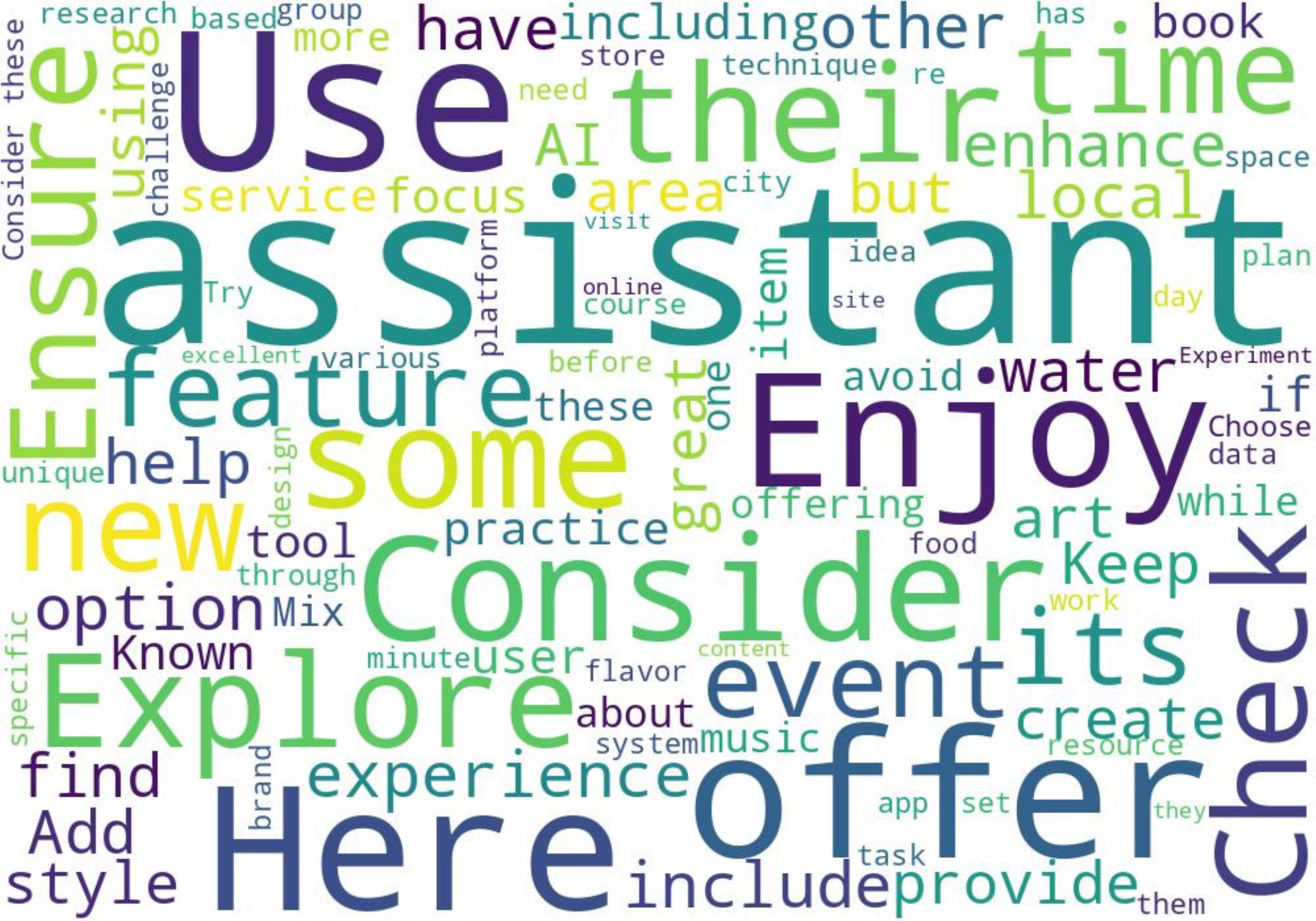}
\caption{{The Word Cloud Figure of General Semantic Memory on LongMemEval-S Dataset.}}
\label{fig:gm_wordcloud}
\end{minipage}
\end{figure}

\subsection{Router Classification Performance}
\label{app:trimem_evaluation}
In this study, we evaluated the classification performance of four large language models (LLMs) on the TriMEM dataset: Qwen3-8b \citep{qwen3}, Gemini-2.5-Flash \citep{gemini2.5}, GPT-4o-Mini \citep{gpt-4o}, and Qwen3-32b with multi-class prompt in Figure \ref{multi_class_prompt}. Additionally, our model was included for comparison. The evaluation metrics used were accuracy (ACC) and F1 score, which respectively measure the overall classification correctness and the balance between precision and recall. Performance was analyzed across three memory dimensions in TriMEM: Episodic Memory (EM), Personal Semantic Memory (PM), and General Semantic Memory (GM), and PerLTQA.

Experimental results indicate that our model outperformed all other LLMs in every category. For overall accuracy and F1 score, our model achieved 92.10\% and 91.75\%, respectively, significantly surpassing GPT-4o-Mini, the second-best model, which scored 83.55\% (ACC) and 76.94\% (F1). In the EM dimension, our model achieved an outstanding 98.80\% (ACC) and 97.81\% (F1), outperforming GPT-4o-Mini, which achieved 85.05\% (ACC) and 76.10\% (F1). For PM, our model reached 82.65\% (ACC) and 85.64\% (F1), while GPT-4o-Mini achieved slightly higher accuracy (87.18\%) but lower F1 (85.59\%). Finally, in GM, our model achieved 94.85\% (ACC) and 92.39\% (F1), outperforming others by a clear margin.

Despite our model's superior performance, the results highlight that existing LLMs struggle to achieve precise classification on TriMEM, requiring further advancements in memory modeling for improved accuracy and generalization.

\subsection{Data Leakage Analysis}

\begin{wrapfigure}{r}{0.5\textwidth}
\centering
    \includegraphics[width=1\linewidth]{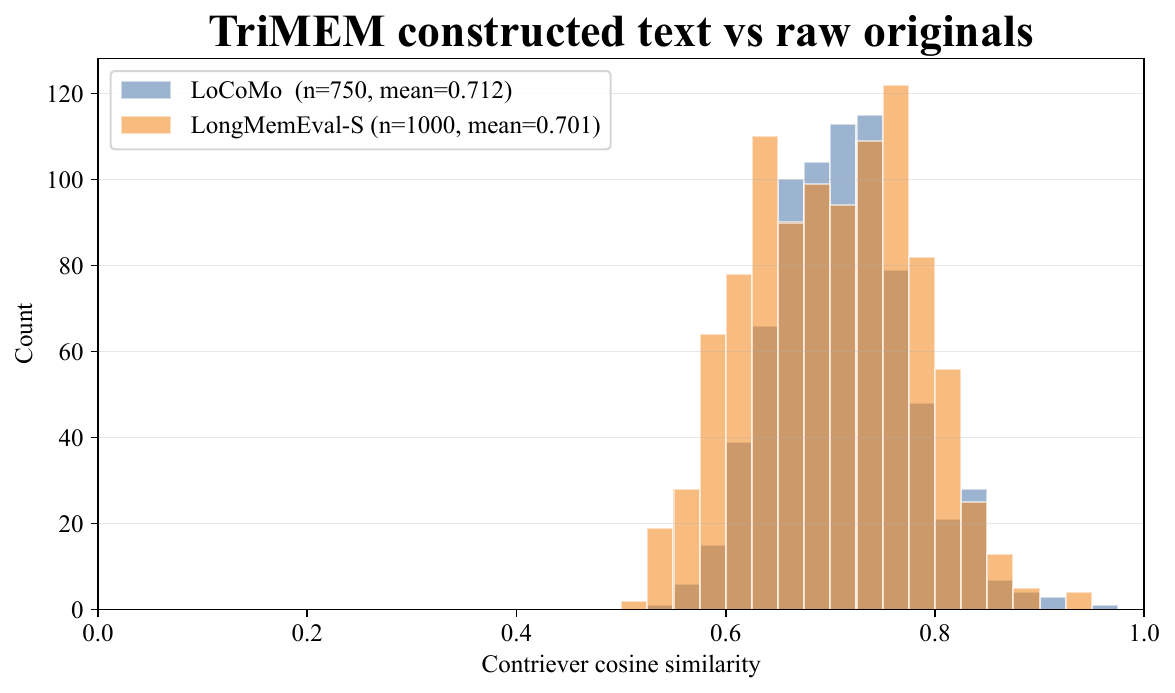}
    
    \caption
          {The Contriever Similarity Distribution between TriMem Context with the raw context on two benchmarks.}
    \label{fig:similarity}
    \vspace{-0.5 em}
\end{wrapfigure}

As shown in Fig.~\ref{fig:optim-step-hparam}, we compare the similarity between the transferred memories produced by \ourmodel and the original memories from the dataset using three surface-form metrics (Word Overlap Rate, BLEU, and 3-gram Overlap Rate) on LongMemEval-S and LoCoMo. Across both datasets, the distributions of all three metrics are concentrated in the low-similarity regime, and most samples fall below the threshold of 0.5. The corresponding frequency distributions in Fig.~\ref{fig:locomo_freq} and Fig.~\ref{fig:longmemeval_freq} further confirm this trend: the mean and median values remain low for each metric, with only a small fraction of cases exhibiting higher overlap.

To further address concerns about \emph{semantic-level} leakage that surface-form metrics cannot capture, we additionally measure semantic similarity using Contriever embeddings, which is the same encoder employed in our retrieval experiments. As shown in Fig.~\ref{fig:similarity}, the Contriever cosine similarity between transferred and original memories is centered around $0.70$ on both LoCoMo (mean $= 0.712$) and LongMemEval-S (mean $= 0.701$), with the distribution concentrated in the $0.6$--$0.8$ range and only a negligible fraction exceeding $0.9$. This level of semantic similarity reflects the \emph{type-level} consistency that is required by design---rewritten dialogues must preserve enough meaning to retain their original memory-type label---rather than instance-level duplication. Notably, the absence of a high-similarity peak ($>0.9$) indicates that the transferred texts are not semantic near-duplicates of their sources.

Taken together, the three surface-form metrics and the Contriever-based semantic analysis provide complementary evidence from two independent perspectives: \ourmodel preserves neither lexical patterns nor instance-level semantic content of the original benchmarks. This property is important for interpreting the role of our router: rather than relying on memorization of dataset-specific expressions or templates, the router is forced to route based on the generalized characteristics of memory types reflected in the transferred content. Consequently, the routing behavior learned by it is better aligned with the intended goal of type-level generalization, instead of reproducing or matching individual dataset instances. This conclusion is further corroborated by the strong out-of-distribution performance of the router on PerLTQA (Table~\ref{trimem_class}), a benchmark that was \emph{never} used in TriMEM construction, where the router still achieves $96$--$100\%$ accuracy across all three categories.
\subsection{Dataset Statistics}
\label{trimem_statistic}

These results indicate that the transferred memories are not near-duplicates of the original texts. In other words, \ourmodel does not preserve substantial lexical or n-gram-level content from the source memories, which provides direct evidence against data leakage through copying. This property is important for interpreting the role of our router: rather than relying on memorization of dataset-specific expressions or templates, the router is forced to route based on the generalized characteristics of memory types reflected in the transferred content. Consequently, the routing behavior learned by it is better aligned with the intended goal of type-level generalization, instead of reproducing or matching individual dataset instances.

"Crucially, our router achieves 96–100\% accuracy on \emph{PerLTQA}, a benchmark whose dialogues were \emph{never used} in TriMEM construction. This out-of-distribution performance demonstrates that the router learns generalizable type-level features rather than benchmark-specific patterns; if the router relied on dataset fingerprints, OOD performance would degrade sharply, which is not observed.

\begin{table}[t]\fontsize{8}{8.5}\selectfont 
\centering
\caption{The classification Performance on TriMEM with four LLM models. ``EM'' denotes Episodic Memory, ``PM'' denotes the Personal Semantic Memory, ``GM'' denotes the General Semantic Memory.}
\label{trimem_class}
\resizebox{\linewidth}{!}{
\renewcommand\arraystretch{1}
\tabcolsep=0.3cm
\scalebox{1.02}{
\begin{tabular}{lcccccccc}
\toprule
\multicolumn{1}{c}{\multirow{2}{*}{Type}} & \multirow{2}{*}{Model} & \multicolumn{4}{c}{TriMEM(In Distribution)}                                                                                                                    & \multicolumn{3}{c}{PerLTQA(Out  of Distribution)}            \\ \cmidrule(lr){3-6} \cmidrule(lr){7-9}
\multicolumn{1}{c}{}                      &                        & \multicolumn{1}{l}{Overall} & \multicolumn{1}{l}{EM} & \multicolumn{1}{l}{PM} & \multicolumn{1}{l}{GM} & Profile & Relationship & Events         \\ \midrule
\multicolumn{9}{c}{ACC (100 \%)}                                                                                                                                                                                                                                                                 \\ \midrule
Prompting                                 & Qwen3-8b               & 71.90                       & 81.65                               & 71.37                                        & 62.68                                       & 88.65               & 57.21               & 46.23          \\
Prompting                                 & Gemini-2.5-Flash       & 77.84                       & 82.47                               & 76.72                                        & 74.33                                       & 100.00              & 75.06               & 91.83          \\
Prompting                                 & GPT-4o-Mini            & \uline{83.55}                 & 85.05                               & \textbf{87.18}                               & \uline{78.42}                                 & 100.00              & 89.47               & 92.00          \\
Prompting                                 & Qwen3-32b              & 77.75                       & \uline{87.82}                         & 80.88                                        & 64.55                                       & 100.00              & 95.64               & 87.55          \\
Training                                  & Ours                   & \textbf{92.10}              & \textbf{98.80}                      & \uline{82.65}                                  & \textbf{94.85}                              & \textbf{100.00}     & \textbf{96.49}      & \textbf{96.94} \\ \midrule
\multicolumn{9}{c}{F1 (100 \%)}                                                                                                                                                                                                                                                                  \\ \midrule
Prompting                                 & Qwen3-8b               & 55.26                       & 49.05                               & 55.58                                        & 61.15                                       & 93.98               & 72.78               & 63.23          \\
Prompting                                 & Gemini-2.5-Flash       & 72.75                       & 73.47                               & 74.55                                        & 70.22                                       & 100.00              & 85.75               & 95.74          \\
Prompting                                 & GPT-4o-Mini            & \uline{76.94}                 & 76.10                               & \uline{85.59}                                  & \uline{69.14}                                 & 100.00              & 94.44               & 95.83          \\
Prompting                                 & Qwen3-32b              & 72.03                       & \uline{76.80}                         & 75.67                                        & 63.61                                       & 100.00              & 98.29               & 93.36          \\
Training                                  & Ours                   & \textbf{91.75}              & \textbf{97.81}                      & \textbf{85.64}                               & \textbf{92.39}                              & \textbf{100.00}     & \textbf{98.45}      & \textbf{98.21} \\ \bottomrule
\end{tabular}}}
\end{table}

\begin{figure}[t]
    \centering
    \begin{subfigure}[t]{0.4\linewidth}
        \centering
        \includegraphics[width=\linewidth]{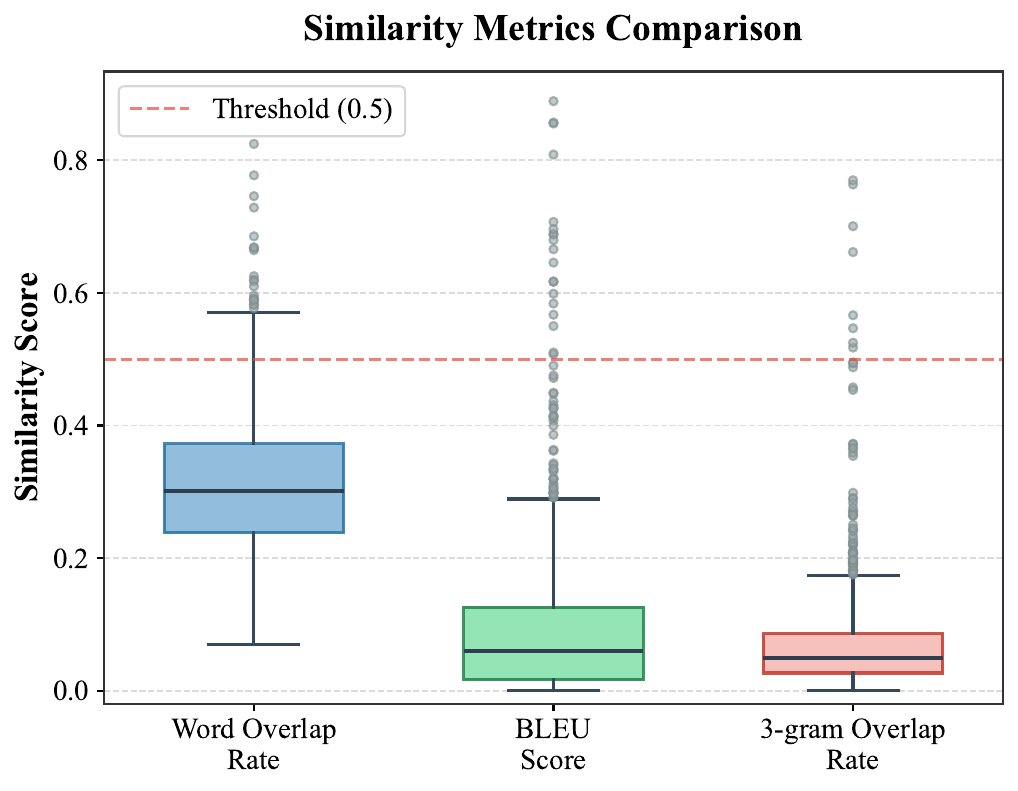}
        \caption{LongMemEval-S}
        \label{fig:longmemeval_compare}
    \end{subfigure}
    \begin{subfigure}[t]{0.4\linewidth}
        \centering
        \includegraphics[width=\linewidth]{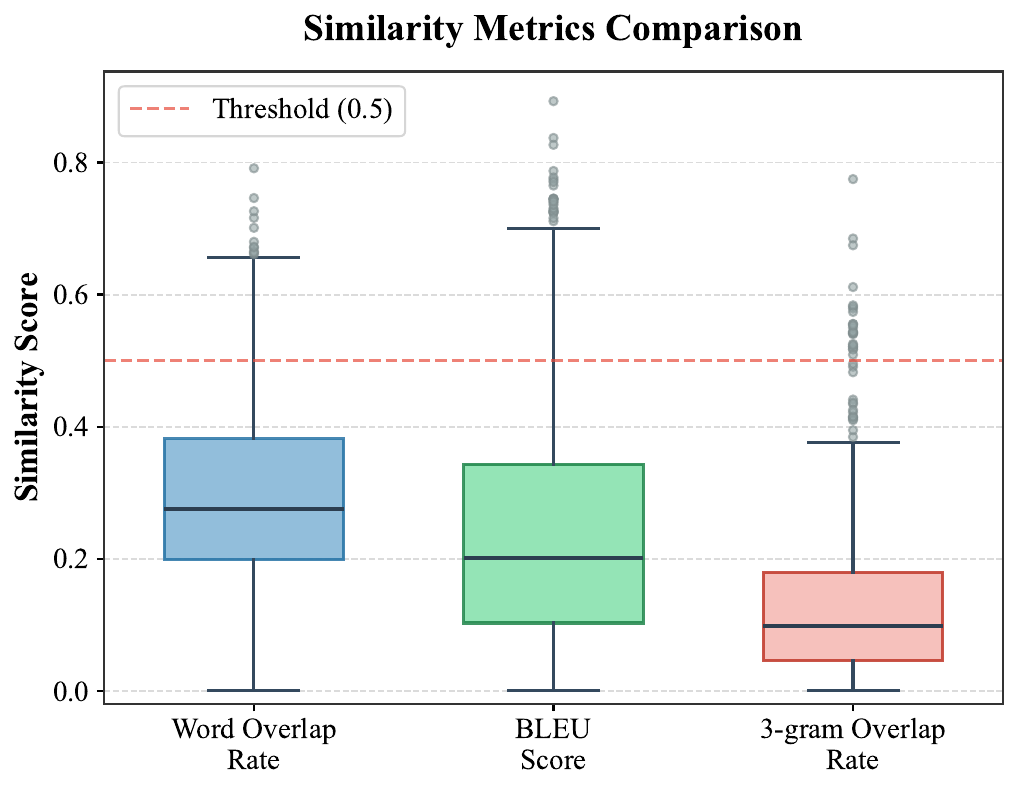}
        \caption{LoCoMo}
        \label{fig:locomo_compare}
    \end{subfigure}
\vspace{0.6em}
    \caption{The Distribution of three Similarity Metrics (Word Overlap Rate, BLEU Score, and 3-gram Overlap Rate) between Transferred Memory and Original Memory on LongMemEval-S and LoCoMo datsets. The red dotted line indicates a similarity threshold of 0.5. The majority of samples fall in the low similarity region, indicating a significant difference between the transformed text and the original text. }
    \label{fig:optim-step-hparam}
\end{figure}
\begin{figure*}[t]
    \centering

    \begin{subfigure}[t]{0.33\linewidth}
        \centering
        \includegraphics[width=\linewidth]{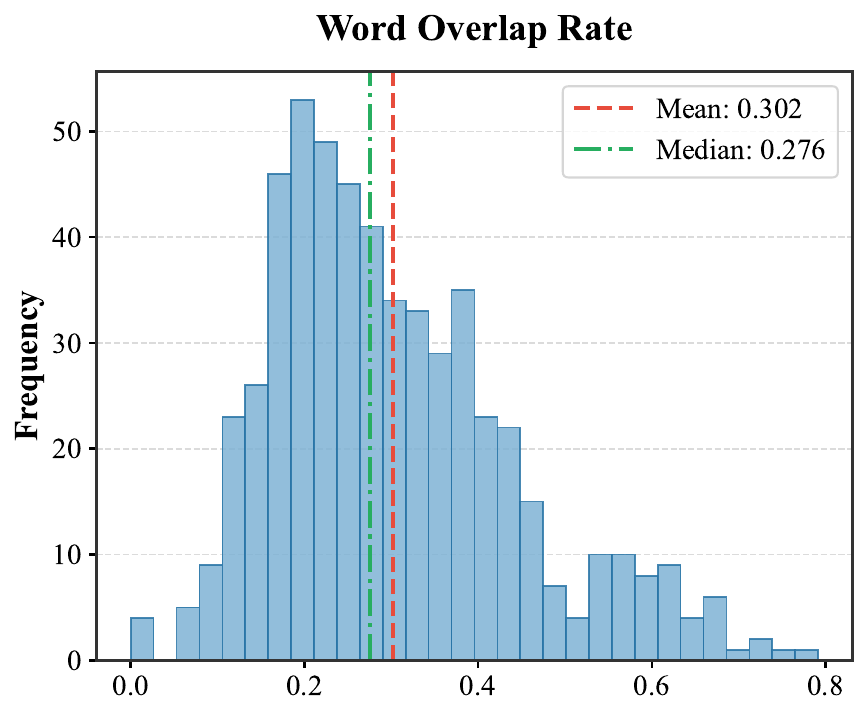}
        \label{fig:locomo_word}
    \end{subfigure}
    \begin{subfigure}[t]{0.32\linewidth}
        \centering
        \includegraphics[width=\linewidth]{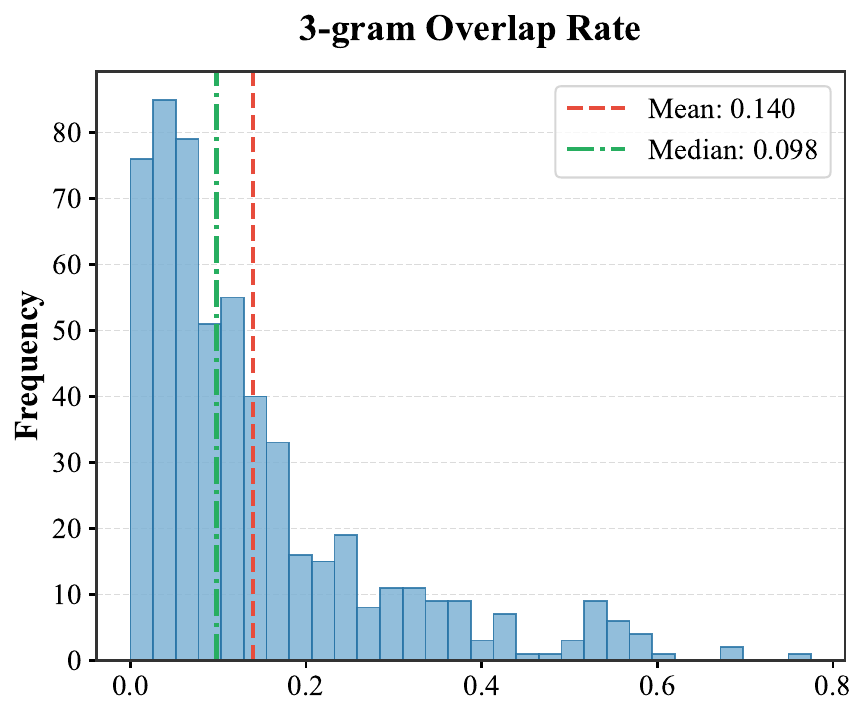}
        \label{fig:locomo_ngram}
    \end{subfigure}
\begin{subfigure}[t]{0.32\linewidth}
        \centering
        \includegraphics[width=\linewidth]{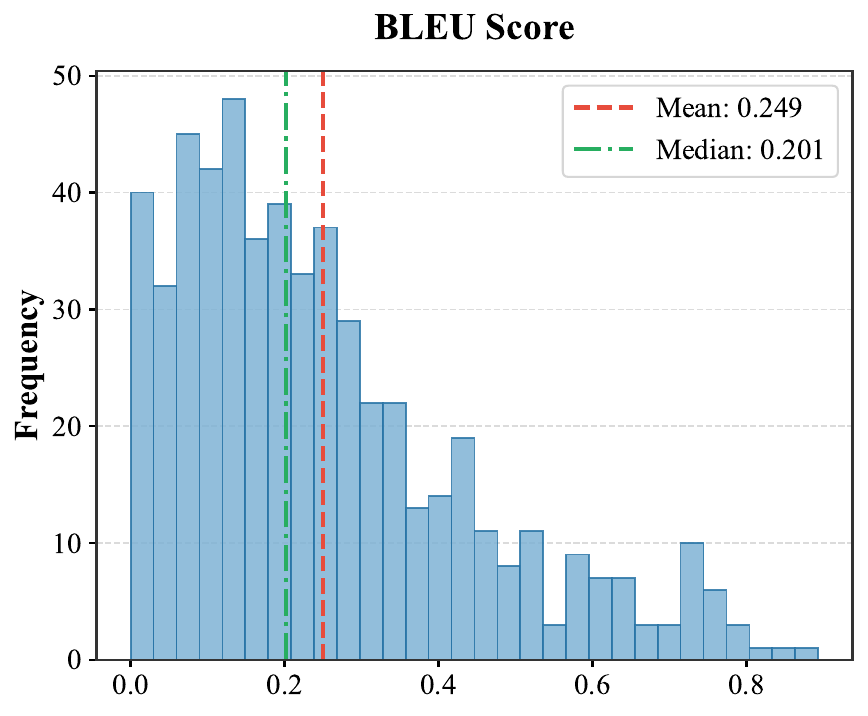}
        \label{fig:locomo_bleu}
    \end{subfigure}
    \caption{The Frequency Distributions of three Similarity Metrics between Transferred Memory and Original Memory on LoCoMo dataset. The red/green dashed lines indicate the mean and median, respectively. The distribution is overall low and long-tailed, further indicating that the transformed memory differs significantly from the original.}
    \label{fig:locomo_freq}
\end{figure*}
\begin{figure*}[t]
    \centering
    \begin{subfigure}[t]{0.32\linewidth}
        \centering
        \includegraphics[width=\linewidth]{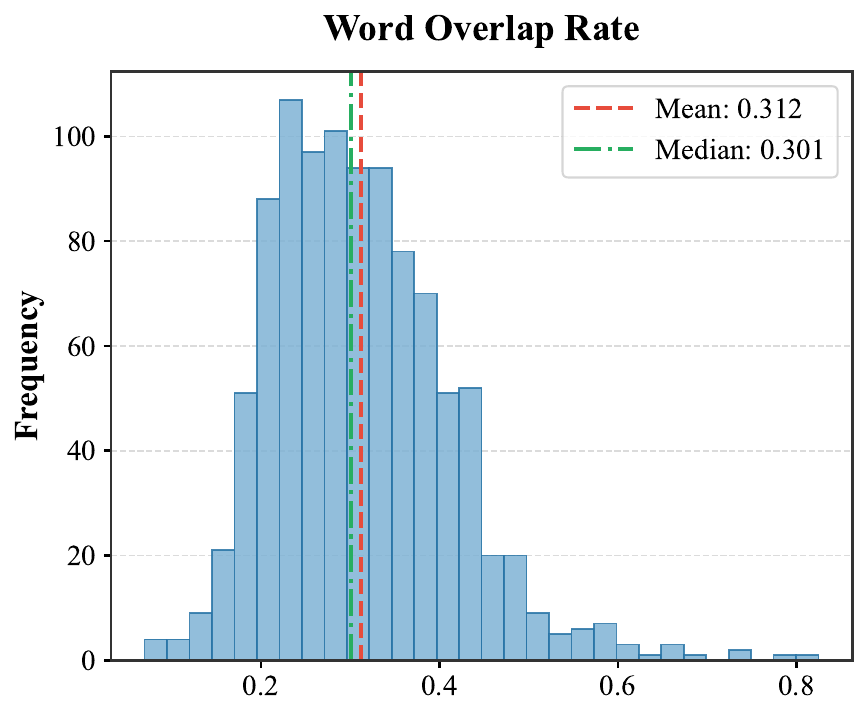}
        \label{fig:longmemeval_word}
    \end{subfigure}
    \begin{subfigure}[t]{0.32\linewidth}
        \centering
        \includegraphics[width=\linewidth]{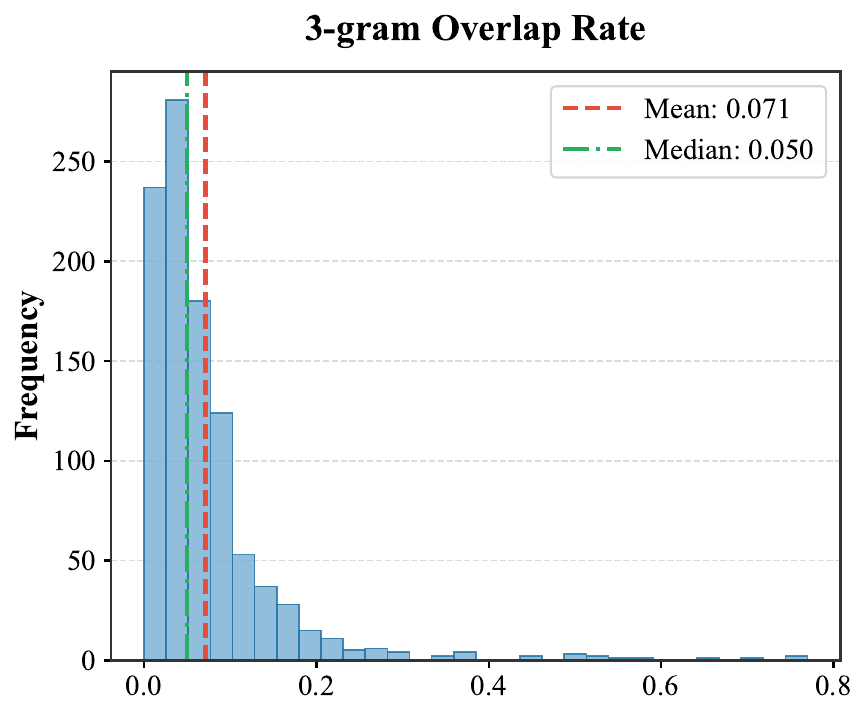}
        \label{fig:longmemeval_ngram}
    \end{subfigure}
\begin{subfigure}[t]{0.32\linewidth}
        \centering
        \includegraphics[width=\linewidth]{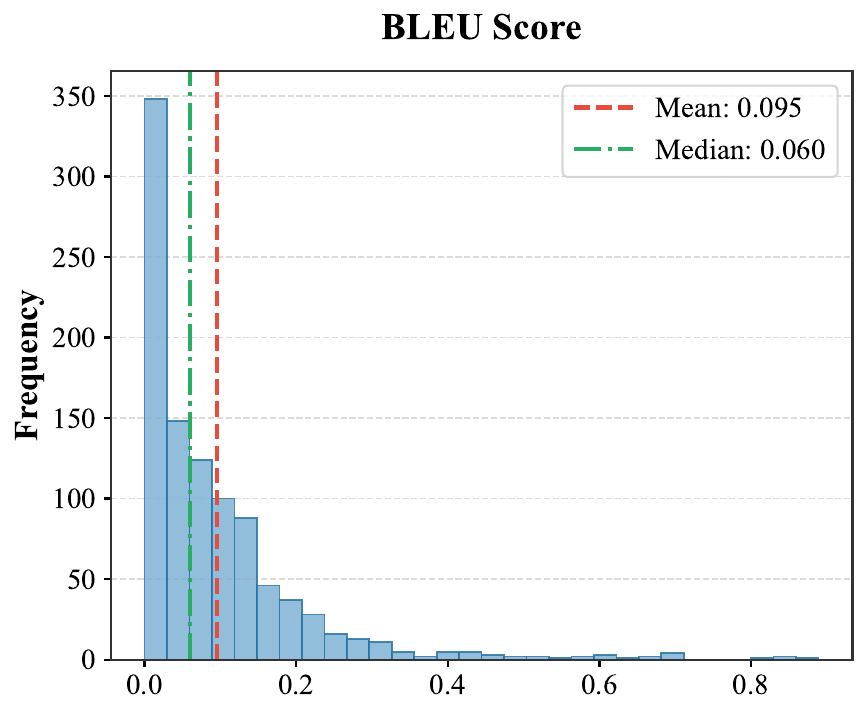}
        \label{fig:longmemeval_bleu}
    \end{subfigure}
    \caption{The Frequency Distributions of three Similarity Metrics between Transferred Memory and Original Memory on LongMemEval-S dataset. The red/green dashed lines indicate the mean and median, respectively. It is evident that all three indicators are concentrated in the lower range.}
    \label{fig:longmemeval_freq}
\end{figure*}
\section{Benchmark and Baseline Details}
\subsection{Evaluation Dataset Details}
\label{dataset}
As statistic information in shown in Table~\ref{tab:dataset-stats}, we evaluate our model with the following datasets:
\begin{itemize}[leftmargin=*]
     \item \textbf{LongMemEval-S} \citep{longmemeval} is a dataset designed to evaluate long-term memory in conversational AI systems. It includes 50 question sessions, each with an average of 115,000 tokens, providing a compact yet challenging benchmark. This dataset tests core memory abilities such as information extraction, multi-session reasoning, temporal reasoning, and abstention. LongMemEval-s requires chat assistants to process user-AI dialogues effectively, retain pertinent information over extended conversation histories, and maintain consistency across sessions. By focusing on shorter yet complex contexts, LongMemEval-s serves as a foundational benchmark for assessing memory performance.
     \item \textbf{LoCoMo} \citep{locomo} is a long-context memory benchmark, evaluates AI systems' ability to handle lengthy dialogues and long-range dependencies. It features conversations with an average of 300 turns, 9,000 tokens, and up to 35 sessions, with a publicly available subset, LoCoMo-10, comprising ten high-quality conversations. The dataset is well-suited for testing long-context language models (LLMs) and retrieval-augmented generation (RAG) systems. Despite advances, systems still struggle with temporal reasoning and maintaining coherence in extended interactions, highlighting the challenge of replicating human-like memory performance.
     \item \textbf{LongMemEval-M} \citep{longmemeval} builds on LongMemEval-S with a more extensive evaluation setting, encompassing 500 sessions per question and an average of 1.5 million tokens per conversation. This dataset is designed for rigorous testing of long-term memory and scalability. It challenges systems to handle dynamic user-AI interactions across an even greater number of sessions, requiring robust memory mechanisms for historical consistency and information recall. The dataset underscores the substantial performance gap between state-of-the-art AI systems and human capabilities in long-term memory retention and reasoning over extended contexts.
     \item \textbf{PerLTQA}~\citep{perltqa} is a benchmark released to study personalized long-term memory in conversational AI. Each instance is built around a virtual persona and bundles four complementary memory views—structured profile, social relationships, life events, and naturalistic multi-turn dialogues—together with a large pool of human-curated question–answer pairs whose evidence is grounded in specific dialogue turns. By unifying semantic memory (facts about the user and their world) with episodic memory (what was actually said), PerLTQA enables joint evaluation of memory classification (which memory type to consult), retrieval (locating the supporting record), and fusion (synthesizing across memories) for downstream QA. Its persona-centric, multi-faceted design makes it complementary to conversation-only benchmarks such as LOCOMO and LongMemEval, and a natural testbed for memory-augmented LLM agents.

\end{itemize}

% 在导言区添加: \usepackage{booktabs}  \usepackage{multirow}

\begin{table}[t]
\centering
\caption{Statistics of the long-term memory QA datasets used in our study.
``\#Inst.'' counts the top-level units. ``Sess./Inst.'' and
``Turns/Sess.'' are averages; ``Words/Inst.'' is the average number of
whitespace-tokenized words per instance.}
\label{tab:dataset-stats}
\setlength{\tabcolsep}{4pt}
\begin{tabular}{lrrrrrrc}
\toprule
\textbf{Dataset} & \textbf{\#Inst.} & \textbf{\#QA} & \textbf{\#Sess.} & \textbf{Sess./Inst.} & \textbf{Turns/Sess.} & \textbf{Words/Inst.} & \textbf{\#Q-Types} \\
\midrule
LoCoMo            & 10  & 1{,}986 & 272      & 27.2            & 21.6 & 15.9\,k   & 5 \\
LongMemEval-S & 500 & 500     & 25{,}097 & 50.2            & \phantom{0}9.8  & 79.0\,k   & 6 \\
LongMemEval-M & 500 & 500     & 249{,}525 & 499.1          & \phantom{0}9.8  & 777.6\,k  & 6 \\
PerLTQA           & 32  & 2{,}838 & 128 & 4.0 & 21.7 & \phantom{0}8.9\,k & 4\\
\bottomrule
\end{tabular}

\vspace{2pt}
\end{table}

\subsection{Baseline Descriptions}
\label{baseline}
We evaluate our model with the following baselines:
\begin{itemize}[leftmargin=*]
\item \textbf{Contriever} \citep{contriever} is a retrieval model that serves as a strong baseline by efficiently retrieving and ranking relevant documents based on dense representations.
\item \textbf{MPNet} \citep{mpnet} is a pre-trained transformer model that combines masked language modeling and permuted language modeling to better capture dependencies between tokens for improved natural language understanding.
\item  \textbf{MiniLM} \citep{minilm} is a lightweight and efficient transformer-based model designed for natural language understanding tasks, offering competitive performance with significantly fewer parameters compared to larger models.

\item \textbf{QAMiniLM} \citep{minilm} is an extension of MiniLM, designed specifically for Question Answering tasks, leveraging MiniLM's lightweight architecture while fine-tuning it to optimize performance for extracting precise answers from text.

\item \textbf{HyDE} \citep{hyde} generates a "hypothetical" document using a language model and retrieves similar real documents via dense embeddings to improve zero-shot dense retrieval without needing labeled relevance data. The prompt of HyDE could be found in Figure \ref{hyde_p}.

\item \textbf{Mill} \citep{mill} utilizes large language models to generate diverse sub-queries and documents, followed by a mutual verification process to synergize generated and retrieved data for effective zero-shot query expansion. The prompt of Mill could be found in Figure \ref{mill_p}.

\item \textbf{Query2Doc} \citep{query2doc} improves sparse and dense retrieval systems by generating pseudo-documents through few-shot prompting of large language models (LLMs) and using them for query expansion, achieving significant performance boosts without fine-tuning. The prompt of QueryDoc could be found in Figure \ref{query2doc_p}.

\item \textbf{Secom} \citep{secom} enhances memory retrieval in long-term conversations by segmenting them into topically coherent units and applying compression-based denoising for improved retrieval accuracy and semantic quality.

\item \textbf{HippoRAG2} \citep{hipporag} is a framework that enhances retrieval-augmented generation (RAG) by integrating deeper passage connections and effective LLM use, achieving superior performance on factual, sense-making, and associative memory tasks, thereby advancing non-parametric continual learning for LLMs.

\item \textbf{Raptor} \citep{raptor}  improves retrieval-augmented language models by recursively embedding, clustering, and summarizing text into a hierarchical tree, enabling multi-level abstraction retrieval and achieving state-of-the-art performance in complex reasoning tasks.

\item \textbf{A-Mem} \citep{a-mem} introduces an agentic memory system for LLMs that dynamically organizes and evolves memories by leveraging Zettelkasten principles, enabling adaptive and context-aware knowledge management through dynamic linking and indexing.

\item \textbf{MemoryOS}~\citep{memoryos} treats LLM-agent memory as an
operating-system problem, organizing it into three tiers, short-term
dialogue pages, mid-term topic segments updated by a heat score
(visits, length, recency), and long-term user/agent personas---and
coordinating Storage, Updating, two-stage segment-then-page Retrieval,
and Generation modules.
\item \textbf{LightMem}~\citep{lightmem} follows the Atkinson--Shiffrin
three-stage view: a sensory stage compresses tokens with LLMLingua-2
and segments turns by topic, a topic-aware short-term stage summarizes
each chunk, and a long-term stage uses fast soft writes online plus an
offline ``sleep-time update'' that merges and de-duplicates entries.
\end{itemize}

\section{Additional Experiments} 
\label{app:experiment}

\begin{table*}[t]\fontsize{8}{8.5}\selectfont 
\centering
\caption{The ablation study of different retrievers.}
\label{table:retriever_ablation}
\vspace{-1 em}
\resizebox{\linewidth}{!}{
\renewcommand\arraystretch{1}
\tabcolsep=0.1cm
\scalebox{1.02}{
\begin{tabular}{ccccccccccccc} \toprule
\multirow{2}{*}{Method} & \multicolumn{3}{c}{LongMemEval-S}                & \multicolumn{3}{c}{LoCoMo}                       & \multicolumn{3}{c}{PerLTQA}                      \\ \cmidrule(lr){2-4} \cmidrule(lr){5-7} \cmidrule(lr){8-10}
                        & Recall@1       & Recall@3       & Recall@5       & Recall@1       & Recall@3       & Recall@5       & Recall@1       & Recall@3       & Recall@5       \\ \midrule
\multicolumn{10}{c}{Base   Retriever:   MPNet}                                                                                                                                   \\ \midrule
HyDE                    & 25.11          & 53.19          & 66.60          & 27.29          & 44.76          & 53.52          & 30.62          & 65.68          & 74.17          \\
Mill                    & 27.45          & 53.19          & 64.89          & 21.00          & 36.81          & 44.51          & 36.15          & 67.27          & 74.77          \\
Query2Doc               & 27.45          & 54.68          & 68.72          & 26.38          & 46.12          & 54.03          & 33.51          & 65.82          & 74.63          \\
SeCom                   & 37.66          & 62.13          & 73.19          & 21.45          & 34.24          & 41.49          & 36.19          & 65.72          & 73.75          \\
A-Mem                   & 27.80          & 56.60          & 67.80          & 29.20          & 49.65          & 58.61          & 32.42          & 66.38          & 74.91          \\
MemoryOS                & 31.52          & 57.20          & 65.55          & 23.21          & 42.45          & 51.26          & 42.07          & 65.12          & 72.23          \\
Ours                    & \textbf{47.87} & \textbf{74.68} & \textbf{83.40} & \textbf{30.87} & \textbf{49.45} & \textbf{58.11} & \textbf{58.00} & \textbf{78.58} & \textbf{85.17} \\ \midrule
\multicolumn{10}{c}{Base   Retriever:   MiniLM}                                                                                                                                  \\ \midrule
HyDE                    & 32.34          & 61.70          & 72.55          & 22.00          & 37.92          & 45.87          & 23.15          & 59.76          & 67.58          \\
Mill                    & 36.81          & 63.62          & 75.11          & 16.26          & 31.27          & 38.97          & 22.45          & 52.43          & 60.50          \\
Query2Doc               & 29.79          & 61.49          & 72.34          & 22.91          & 38.32          & 46.78          & 24.77          & 59.76          & 66.98          \\
SeCom                   & 37.02          & 64.04          & 77.45          & 16.11          & 28.54          & 34.74          & 23.08          & 53.49          & 61.17          \\
A-Mem                   & 35.40          & 63.60          & 73.60          & 21.25          & 37.81          & 46.32          & 31.15          & 62.65          & 71.49          \\
MemoryOS                & 36.95          & 60.54          & 70.98          & 20.34          & 34.39          & 41.49          & 29.25          & 53.07          & 59.16          \\
Ours                    & \textbf{50.21} & \textbf{77.66} & \textbf{84.68} & \textbf{28.30} & \textbf{45.67} & \textbf{53.47} & \textbf{49.01} & \textbf{71.35} & \textbf{79.63} \\ \midrule
\multicolumn{10}{c}{Base   Retriever:   QAMiniLM}                                                                                                                                \\ \midrule
HyDE                    & 31.91          & 61.06          & 73.19          & 17.27          & 30.56          & 37.61          & 30.62          & 65.68          & 74.17          \\
Mill                    & 35.32          & 63.19          & 74.04          & 9.42           & 18.98          & 26.13          & 36.15          & 67.27          & 74.77          \\
Query2Doc               & 33.62          & 61.06          & 75.11          & 19.08          & 32.63          & 40.94          & 33.51          & 65.82          & 74.63          \\
SeCom                   & 41.06          & 71.06          & 81.28          & 15.56          & 28.80          & 36.60          & 23.82          & 55.29          & 63.57          \\
A-Mem                   & 32.40          & 61.00          & 71.60          & 20.49          & 37.51          & 45.17          & 30.83          & 62.68          & 70.75          \\
MemoryOS                & 37.37          & 59.71          & 69.31          & 22.81          & 37.66          & 44.91          & 42.07          & 65.12          & 72.23          \\
Ours                    & \textbf{49.79} & \textbf{76.17} & \textbf{84.47} & \textbf{27.14} & \textbf{44.86} & \textbf{52.87} & \textbf{50.39} & \textbf{72.83} & \textbf{80.23}  \\ \bottomrule
\end{tabular}}}
\end{table*}

\subsection{Retriever Ablation}
\label{app:retriever_ablation}

Our main experiments adopt Contriever as the default retriever. A natural question is whether the performance gains of MemoType are specific to one particular retriever or generalize across different embedding models. To investigate this, we conduct retriever ablation experiments by replacing Contriever with three alternative retrievers: MPNet, MiniLM, and QAMiniLM, covering both general-purpose and task-specific encoders.

As shown in Table~\ref{table:retriever_ablation}, we evaluate MemoType against six baselines across three datasets under each retriever. Across all retriever and dataset combinations, MemoType consistently achieves the highest recall scores at ranks 1, 3, and 5. For instance, with MPNet on LongMemEval-S, MemoType achieves a Recall@1 of 47.87, outperforming the strongest baseline SeCom by over 10 points. With MiniLM on PerLTQA, MemoType reaches 49.01 in Recall@1, surpassing A-Mem by nearly 18 points. Similar trends hold for QAMiniLM, where MemoType achieves 49.79 Recall@1 on LongMemEval-S, exceeding SeCom by 8.73 points.

These results reveal two important findings. First, the type-aware retrieval strategy provides consistent improvements regardless of the underlying embedding model, confirming that the gains stem from the diversity-aware routing mechanism rather than a favorable interaction with a specific encoder. Second, the improvement margins remain substantial even with weaker retrievers, suggesting that MemoType effectively compensates for encoder limitations by directing each query to the most relevant memory subset with a tailored strategy. This demonstrates the generalizability and robustness of our framework.

\begin{table*}[t]\fontsize{8}{8.5}\selectfont 
\centering
\caption{The ablation study of the proposed classification strategy on four datasets. ``w/ Classify'' represents adopting the classification strategy to filter memory, ``w/o Classify'' denotes the opposite.}
\label{table:router_ablation}
\resizebox{\linewidth}{!}{
\renewcommand\arraystretch{1}
\tabcolsep=0.2cm
\scalebox{1.02}{
\begin{tabular}{ccccccccc}
\toprule
Method       & Recall@1       & NDCG@1         & Recall@3       & NDCG@3         & Recall@5       & NDCG@5         & Recall@10      & NDCG@10        \\ \midrule
\multicolumn{9}{c}{LongMemEval-S}                                                                                                                    \\ \midrule
w/o Classify & 43.40          & 43.40          & 73.62          & 54.00          & 83.83          & 60.33          & \textbf{94.26} & 65.79          \\
w/ Classify  & \textbf{50.21} & \textbf{50.21} & \textbf{79.79} & \textbf{60.32} & \textbf{86.81} & \textbf{65.47} & 92.34          & \textbf{70.15} \\ \midrule
\multicolumn{9}{c}{LoCoMo}                                                                                                                           \\ \midrule
w/o Classify & 20.75          & 20.75          & 36.56          & 30.86          & 44.06          & 34.16          & 53.93          & 37.31          \\
w/ Classify  & \textbf{25.13} & \textbf{25.13} & \textbf{41.99} & \textbf{35.88} & \textbf{49.70} & \textbf{39.06} & \textbf{59.47} & \textbf{42.36} \\ \midrule
\multicolumn{9}{c}{LongMemEval-M}                                                                                                                    \\ \midrule
w/o Classify & 32.77          & 32.77          & 57.45          & 40.85          & 67.45          & 46.22          & 80.34          & 51.50          \\
w/ Classify  & \textbf{37.45} & \textbf{37.45} & \textbf{63.19} & \textbf{44.37} & \textbf{73.83} & \textbf{50.41} & \textbf{80.64} & \textbf{54.52} \\ \midrule
\multicolumn{9}{c}{PerLTQA}                                                                                                                          \\ \midrule
w/o Classify & 55.32          & 55.32          & 80.44          & 77.88          & 86.54          & 80.75          & 92.35          & 82.79          \\
w/ Classify  & \textbf{61.98} & \textbf{61.98} & \textbf{82.45} & \textbf{80.42} & \textbf{88.09} & \textbf{83.09} & \textbf{93.80} & \textbf{85.07}     \\ \bottomrule            
\end{tabular}}}
\end{table*}

\subsection{Router Performance Ablation} \label{app:router_ablation}

A core design of MemoType is the memory type router that classifies both memories and queries into distinct types, enabling type-aware retrieval within a filtered subset rather than the full corpus. To verify whether this routing mechanism genuinely contributes to retrieval quality, we compare two settings: ``w/ Classify'', which filters the memory corpus based on predicted query types before retrieval, and ``w/o Classify'', which retrieves from the entire unfiltered corpus. Importantly, no retrieval enhancement strategies are applied in either setting, isolating the effect of classification filtering from other components.

As shown in Table~\ref{table:router_ablation}, the classification strategy yields consistent improvements across all four datasets. On LongMemEval-S, filtering improves Recall@1 from 43.40 to 50.21 and NDCG@5 from 60.33 to 65.47, indicating that narrowing the search space helps the retriever focus on semantically relevant candidates. On LoCoMo, the gains are even more pronounced, with Recall@3 increasing from 36.56 to 41.99 and NDCG@10 rising by 5.05 points. On LongMemEval-M, the most challenging dataset with nearly 500 sessions per question, Recall@1 improves by 4.68 points, demonstrating that type-aware filtering becomes increasingly valuable as the corpus scale grows. On PerLTQA, Recall@1 increases from 55.32 to 61.98 with consistent gains across all metrics.

These results confirm that the router effectively reduces retrieval noise by eliminating irrelevant memory categories from the candidate pool. The improvement is achieved purely through corpus filtering without any query or memory augmentation, validating that accurate memory classification alone provides a strong foundation for downstream retrieval.

\begin{table*}[t]\fontsize{8}{8.5}\selectfont 
\centering
\caption{The ablation study of memory pruning strategy on three datasets. ``w/ Prun'' denotes adopting the proposed memory pruning strategy,  and``w/o Prun'' denotes the contrary.}
\label{table:prun}
\resizebox{\linewidth}{!}{
\renewcommand\arraystretch{1}
\tabcolsep=0.25cm
\scalebox{1.02}{ \begin{tabular}{ccccccccccccc} \toprule Strategy & GPT4Judge      & F1             & BLEU           & Rouge1         & Rouge2         & RougeL         & RougeLsum      & BERTScore      \\ \midrule
\multicolumn{9}{c}{LongMemEval-S}                                                                                                                \\ \midrule
w/o Prun & 55.60          & 19.55          & 4.48           & 21.32          & 10.17          & 19.54          & 19.70          & 85.03          \\
w/ Prun  & \textbf{56.60} & \textbf{20.31} & \textbf{4.91}  & \textbf{22.18} & \textbf{10.70} & \textbf{20.28} & \textbf{20.32} & \textbf{85.22} \\ \midrule
\multicolumn{9}{c}{LoCoMo}                                                                                                                       \\ \midrule
w/o Prun & 48.39          & 20.30          & 4.11           & 21.36          & 10.63          & 19.93          & 19.91          & 85.55          \\
w/ Prun  & \textbf{49.19} & \textbf{22.23} & \textbf{4.64}  & \textbf{23.28} & \textbf{11.83} & \textbf{21.83} & \textbf{21.86} & \textbf{85.80} \\ \midrule
\multicolumn{9}{c}{PerLTQA}                                                                                                                      \\ \midrule
w/o Prun & 62.58          & 45.69          & 19.27          & 49.61          & 31.84          & 43.23          & 43.24          & 91.10          \\
w/ Prun  & \textbf{62.83} & \textbf{47.26} & \textbf{20.65} & \textbf{51.20} & \textbf{33.02} & \textbf{44.78} & \textbf{44.79} & \textbf{91.61} \\ \midrule
\multicolumn{9}{c}{LongMemEval-M}                                                                                                                \\ \midrule
w/o Prun & 49.60          & 17.78          & 4.05           & 19.25          & 8.61           & 17.41          & 17.50          & 84.62          \\
w/ Prun  & \textbf{50.20} & \textbf{18.77} & \textbf{4.24}  & \textbf{20.36} & \textbf{9.19}  & \textbf{18.56} & \textbf{18.64} & \textbf{85.02} \\ \bottomrule \end{tabular}
}}
\end{table*}

\subsection{Memory Pruning Ablation} 
\label{app:pruning_ablation}
While the main paper presents the pruning ablation on F1 for two datasets, we provide comprehensive results across all four datasets and all generation metrics in this section. The motivation is to examine whether the query-conditioned pruning module consistently benefits generation quality, or whether its effect is limited to specific metrics or datasets.

As shown in Table~\ref{table:prun}, applying the pruning module yields consistent improvements across all datasets and nearly all metrics. On LoCoMo, the pruning strategy improves F1 from 20.30 to 22.23 and Rouge1 from 21.36 to 23.28, representing gains of approximately 1.93 and 1.92 points respectively. On PerLTQA, BLEU increases from 19.27 to 20.65 and RougeL improves from 43.23 to 44.78. On LongMemEval-M, pruning boosts F1 by nearly 1 point and RougeL by 1.15 points. Even on LongMemEval-S, where the baseline performance is already strong, pruning still provides noticeable gains across all metrics.

The consistent improvements across diverse metrics and datasets confirm that retrieved memory segments inevitably contain information irrelevant to the query, which introduces noise during response generation. The pruning module addresses this by adaptively retaining only query-relevant content before feeding memories into the LLM. Notably, the gains are more substantial on lexical overlap metrics such as BLEU and Rouge compared to BERTScore, suggesting that pruning primarily helps by removing distracting content that causes the model to generate off-topic tokens, while the overall semantic coherence is less affected.

\subsection{Per-Type Retrieval Strategy Contributions}
\label{app:per_type_strategy}
To investigate the independent contribution of each memory-type-specific retrieval strategy, we conduct fine-grained ablation experiments where only one strategy is activated at a time. Specifically, ``Pure'' denotes the baseline with classification filtering but no retrieval enhancement, ``w/ GM'' applies only the general semantic memory strategy, ``w/ EM'' applies only the episodic memory strategy, ``w/ PM'' applies only the personal semantic memory strategy, and ``w/ All'' activates all three strategies simultaneously.

As shown in Figure~\ref{fig:strategy_recall@1} and Figure~\ref{fig:strategy_recall@3}, each individual strategy contributes positively to retrieval performance across all three datasets. Among the three strategies, the episodic memory strategy consistently yields the largest gains. On LoCoMo, w/ EM improves Recall@1 from 25.13 to 28.25, a gain of 3.12 points, while on LongMemEval-S, it achieves 53.83 compared to the Pure baseline of 50.21. This aligns with our expectation, as episodic memory involves temporal reasoning and structured event elements that benefit most from element-decomposed matching. The personal semantic memory strategy also provides substantial improvements, particularly on PerLTQA where Recall@1 increases from 61.98 to 67.62, reflecting the prevalence of user-specific queries in this dataset. The general semantic memory strategy yields relatively modest gains, which is expected since standard embedding similarity already handles factual knowledge reasonably well.

Crucially, combining all three strategies consistently outperforms any individual one, confirming that the strategies are complementary rather than redundant. The full combination achieves the highest scores across all datasets and metrics, validating the design principle of assigning tailored retrieval strategies to different memory types.

\subsection{Router Architecture Ablation}
\label{app:router_architecture}

Our default memory type router is built on Qwen3-1.7B fine-tuned with LoRA, which naturally raises the question of whether the strong routing performance stems from the model's inherent reasoning capacity or from the quality of the TriMEM training data. To disentangle these two factors, we replace Qwen3-1.7B with three lightweight encoder-only models, BERT, DeBERTa, and RoBERTa, all fine-tuned on the same TriMEM dataset, and evaluate the end-to-end retrieval performance.

As shown in Table~\ref{table:bert_ablation}, all three encoder-based routers achieve competitive results that are remarkably close to the Qwen3-1.7B router. On LongMemEval-S, BERT achieves a Recall@1 of 55.74, only 3.41 points below Qwen3-1.7B, while RoBERTa narrows this gap further to just 1.06 points with a Recall@1 of 58.09. On LoCoMo, BERT reaches 26.54 in Recall@1, and RoBERTa achieves 27.49, both delivering substantial improvements over the baselines reported in the main experiments. A similar pattern is observed on LongMemEval-M, where BERT attains 46.17 in Recall@1 and RoBERTa reaches 49.15, even slightly surpassing Qwen3-1.7B on this dataset.

These results lead to two important conclusions. First, the effectiveness of MemoType's routing mechanism is primarily driven by the quality and annotation precision of the TriMEM dataset rather than the scale or reasoning ability of the router model. Even BERT, with only 110M parameters and no generative capability, can serve as a competent router when trained on well-annotated data. Second, since encoder-only models are significantly faster at inference than generative LLMs, practitioners can opt for lightweight routers such as RoBERTa to achieve near-optimal performance with substantially lower latency, making MemoType more practical for deployment in real-world applications where efficiency is a critical concern.
\begin{figure}
    \centering
    \includegraphics[width=\linewidth]{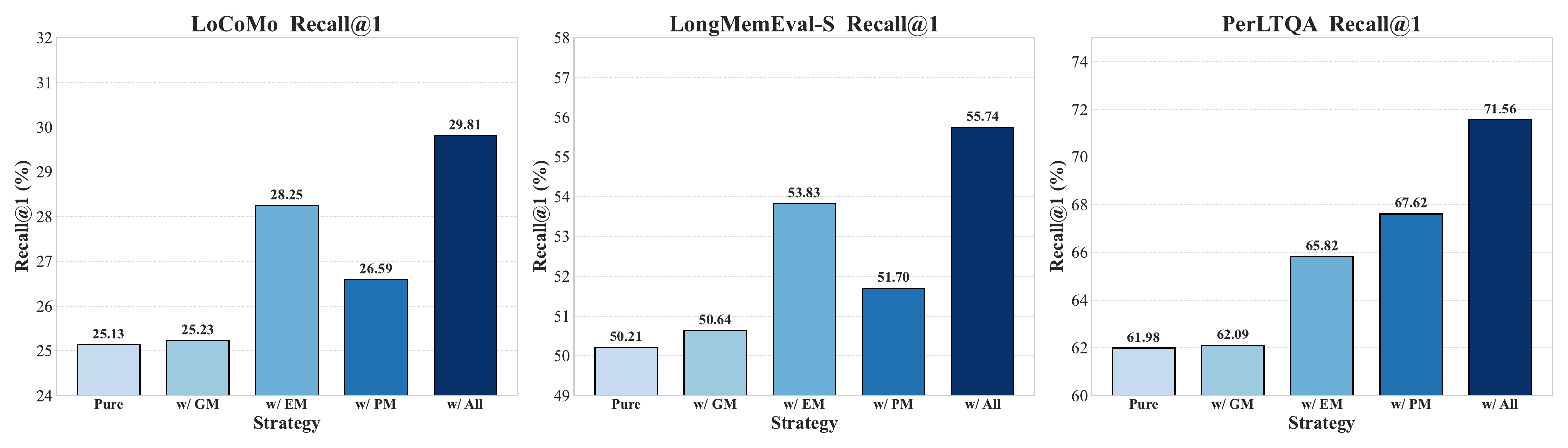}
    \caption{Strategy Ablation on Retrieval Recall@1 across Three Benchmarks.}
    \label{fig:strategy_recall@1}
\end{figure}

\begin{figure}
    \centering
    \includegraphics[width=\linewidth]{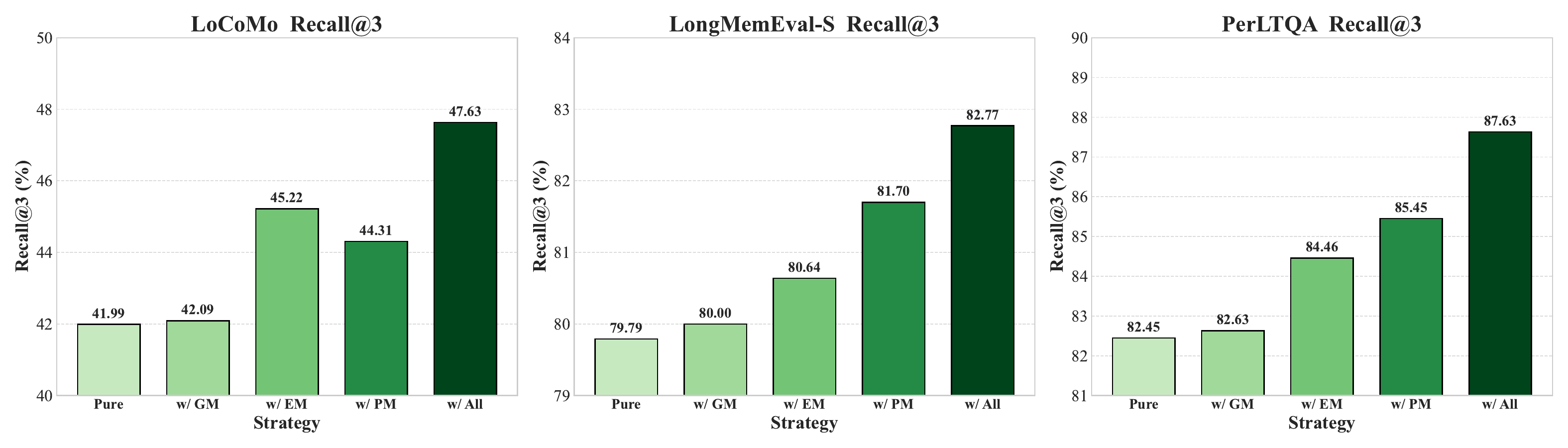}
    \caption{Strategy Ablation on Retrieval Recall@3 across Three Benchmarks.}
    \label{fig:strategy_recall@3}
\end{figure}

\begin{table*}[t]
\caption{{The Ablation Study of BERT Encoder.}}
\label{table:bert_ablation}
\resizebox{\linewidth}{!}{
\renewcommand\arraystretch{1}
\tabcolsep=0.1cm
\scalebox{1.02}{
\begin{tabular}{ccccccccc}
\toprule
{Method}  & {Recall@1} & {NDCG@1} & {Recall@3} & {NDCG@3} & {Recall@5} & {NDCG@5} & {Recall@10} & {NDCG@10}
                                                                     \\ \midrule
 \multicolumn{9}{c}{{LongMemEval-S}}       \\ \midrule
{BERT}    & {55.74}    & {55.74}  & {81.06}    & {63.40}  & {88.72}    & {68.64}  & {93.62}     & {72.54}   \\
{DeBERTa} & {55.32}    & {55.32}  & {81.91}    & {63.19}  & {89.36}    & {68.90}  & {95.74}     & {72.62}   \\
{RoBERTa} & {58.09}    & {58.09}  & {83.83}    & {65.27}  & {90.21}    & {70.70}  & {95.53}     & {74.41}   \\ Qwen3-1.7B & 59.15    & 59.15  & 83.19    & 66.74  & 90.21    & 71.60  & 92.77                                  & 74.85                                \\ \midrule
\multicolumn{9}{c}{{LoCoMo}}                                                                                                                                                                                                                                                        \\ \midrule
{BERT}    & {26.54}    & {26.54}  & {43.15}    & {36.77}  & {50.60}    & {39.79}  & {61.48}     & {43.56}   \\
{DeBERTa} & {26.74}    & {26.74}  & {44.16}    & {38.08}  & {51.71}    & {41.15}  & {62.08}     & {44.56}   \\
{RoBERTa} & {27.49}    & {27.49}  & {43.96}    & {37.46}  & {51.31}    & {40.52}  & {62.13}     & {44.27}   \\ Qwen3-1.7B & 31.07    & 31.07  & 47.73    & 40.97  & 54.08    & 43.66  & 65.01                                  & 47.12                                \\ \midrule
\multicolumn{9}{c}{{LongMemEval-M}}                                                                                                                                                                                                                                                 \\ \midrule
{BERT}    & {46.17}    & {46.17}  & {71.91}    & {52.95}  & {80.21}    & {58.35}  & {87.45}     & {62.29}   \\
{DeBERTa} & {46.38}    & {46.38}  & {68.94}    & {49.73}  & {77.66}    & {55.55}  & {86.81}     & {59.91}   \\
{RoBERTa} & {49.15}    & {49.15}  & {69.57}    & {51.96}  & {79.79}    & {57.52}  & {88.51}     & {62.12}   \\ Qwen3-1.7B & 48.09    & 48.09  & 70.43    & 52.20  & 78.09    & 57.09  & 84.89                                  & 61.27 \\ \bottomrule
\end{tabular}}}
\end{table*}

\subsection{Cost Analysis}
\label{sec:cost_indexing}
We compare the LLM cost of memory construction on four
benchmarks. A-Mem invokes the LLM twice per turn. HippoRAG2 invokes the LLM twice per chunk for
named-entity and triple extraction; since each turn becomes a chunk, the
cost grows linearly with the corpus. MemoType classifies every turn with
a local Qwen3-1.7B router and only invokes the paid LLM on turns routed
as Episodic Memory (EM) for event extraction, so its API cost scales with
$|\text{EM}|$ rather than with the full turn count.
Table~\ref{tab:cost_indexing} reports the result. The fraction of EM
turns is $29.9\%$ on LoCoMo, $8.1\%$ on LongMemEval-S, $2.2\%$ on
LongMemEval-M, and $80.9\%$ on PerLTQA. MemoType is consistently the
cheapest across all four datasets, and the gap widens as the corpus
grows: on LongMemEval-M it reduces the API bill by $146\times$ over
HippoRAG2 and $248\times$ over A-Mem. Even on PerLTQA, where most turns
are EM and the gap is smallest, MemoType still halves the cost of
HippoRAG2 and is more than $3\times$ cheaper than A-Mem. The local
Qwen3-1.7B router is not billed but listed for transparency.
\begin{table*}[t]
\centering
\caption{The cost analysis across four benchmarks. ``Calls'' is the number of paid LLM calls during memory ingestion. Due to the time limits,
$^{\dagger}$HippoRAG2 on LongMemEval-M  is linearly projected from $5/500$ samples.}
\label{tab:cost_indexing}
\resizebox{\linewidth}{!}{
\renewcommand\arraystretch{1}
\tabcolsep=0.1cm
\scalebox{1.02}{
\begin{tabular}{llrrrrrr}
\toprule
Dataset & Method & Calls & Calls/sample & Prompt (M) & Compl. (M) & Total (M) & USD \\
\midrule
\multirow{3}{*}{LoCoMo}
  & A-Mem            & 11{,}764  & 1{,}176.4 & 5.73   & 1.82   & 7.56   & \$1.95 \\
  & HippoRAG2        & 11{,}760  & 1{,}176.0 & 3.97   & 0.33   & 4.30   & \$0.79 \\
  & \textbf{MemoType}& 1{,}756   & 175.6     & 0.29   & 0.08   & 0.37   & \textbf{\$0.09} \\
\midrule
\multirow{3}{*}{LongMemEval-S}
  & A-Mem            & 493{,}574 & 987.1     & 326.67 & 76.50  & 403.18 & \$94.90 \\
  & HippoRAG2        & 492{,}146 & 984.3     & 254.34 & 29.21  & 283.56 & \$55.68 \\
  & \textbf{MemoType}& 19{,}897  & 39.8      & 3.78   & 1.43   & 5.21   & \textbf{\$1.42} \\
\midrule
\multirow{3}{*}{LongMemEval-M}
  & A-Mem                    & 4{,}875{,}958          & 9{,}751.9 & 3{,}222.93 & 755.77 & 3{,}978.70 & \$936.90 \\
  & HippoRAG2$^{\dagger}$    & $\approx$4{,}909{,}400 & 9{,}818.8 & 2{,}523.91 & 283.99 & 2{,}807.90 & \$548.98 \\
  & \textbf{MemoType}        & 52{,}434               & 104.9     & 10.12      & 3.75   & 13.87      & \textbf{\$3.77} \\
\midrule
\multirow{3}{*}{PerLTQA}
  & A-Mem            & 3{,}500   & 109.4     & 2.27   & 0.54   & 2.81   & \$0.67 \\
  & HippoRAG2        & 3{,}488   & 109.0     & 1.76   & 0.30   & 2.05   & \$0.44 \\
  & \textbf{MemoType}& 1{,}416   & 44.2      & 0.49   & 0.19   & 0.68   & \textbf{\$0.19} \\
\bottomrule
\end{tabular}}}
\vspace{2pt}
\footnotesize
\end{table*}

\subsection{Empirical Evidence for Semantic Clustering}
\label{sec:empirical_clustering}

A central premise of Theorem~\ref{thm:precision} is the angular separation assumption (Assumption~\ref{assump:separation}): different memory categories occupy distinct directions in the embedding space. To verify that this assumption is not merely a theoretical convenience but reflects a real property of dialogue memory corpora, we conduct a t-SNE visualization on two benchmarks. Specifically, we randomly sample 500 turns from each of the three memory categories (Episodic Memory, Personal Semantic Memory, and General Semantic Memory) on LongMemEval-S and PerLTQA, encode them with Contriever, and project the resulting embeddings into two dimensions via t-SNE~\citep{wxh_fcmvc_iv, wxh_omvcdr}.

As shown in Fig.~\ref{fig:semantic_clustering}, the three memory types form clearly identifiable clusters on both benchmarks, with each category concentrated in a distinct region of the embedding space. The separation is particularly pronounced between General Semantic Memory and the other two categories, while Episodic Memory and Personal Semantic Memory also exhibit visible boundaries despite their finer-grained semantic distinction. This empirical pattern directly corroborates Assumption~\ref{assump:separation}: the centroids of different categories point in non-identical directions, with non-trivial angular separation $\delta > 0$.

Consequently, the upper bound established in Theorem~\ref{thm:precision} is not vacuous on real corpora---a single shared encoder genuinely faces the geometric tension described in §3, motivating the type-aware retrieval design that processes each cluster through a tailored strategy.

\begin{figure}[t]
\centering
    \includegraphics[width=0.8\linewidth]{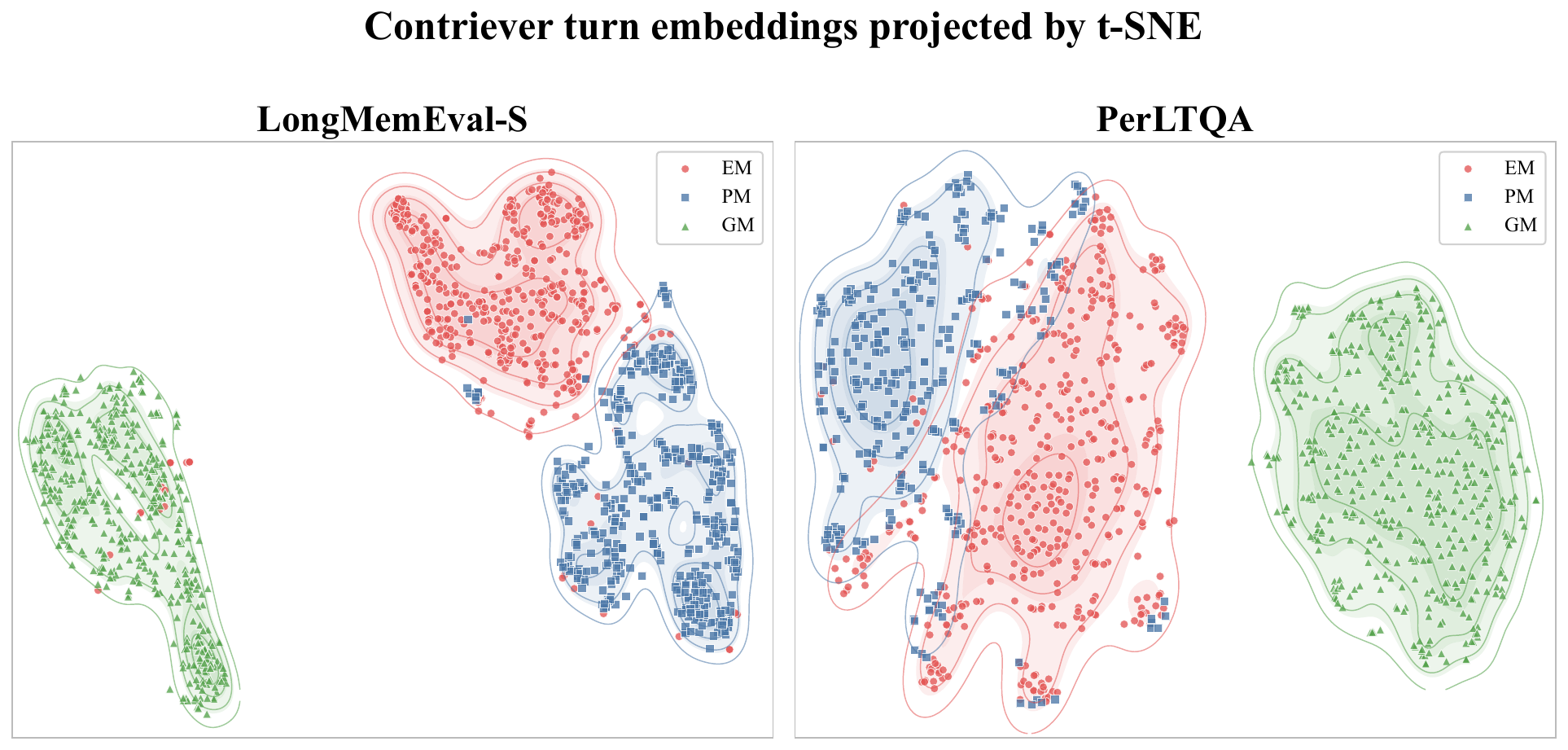}
    
    \caption
          {The Semantic Clustering Evidence with Contriever Embedding on two Benchmarks.}
    \label{fig:semantic_clustering}
    \vspace{-0.5 em}
\end{figure}

\section{Extended Related Work}
\subsection{RAG-based Memory Methods}
A growing body of work equips LLM agents~\citep{xdr_harnessing, jy_cocoa, jy_gainrag, zs_memsearch} with an external memory store
that is written, indexed, and retrieved at inference time, effectively
casting long-term memory as a retrieval-augmented generation (RAG)
problem. Early systems instantiate memory as a vector database of past
turns or summaries and recall semantically relevant entries to condition
generation, often with handcrafted forgetting or salience
heuristics~\citep{recomp, COMEDY, THEANINE, kim2024theanine, personamem, personamem2,memochat,longmem, zyy_rf_mem, memgas, lsc_perma}. To better organize
what is stored, knowledge-centric methods build structured notes or
graphs that link concepts and reasoning traces across sessions, ranging
from cached chains-of-thought~\citep{tim} and agentic note
networks~\citep{a-mem} to neurobiologically inspired hippocampal
indexes~\citep{hipporag}.
A second line treats memory as a tiered architecture inspired by virtual
memory or human cognition, paging information between fast working
context and slower long-term stores under explicit
read/write controls~\citep{memgpt, memos, memoryos,lightmem, bae2022keep, mem0}. Reflection-
and consolidation-based approaches instead periodically rewrite memory:
generative agents distill experience into higher-level
reflections~\citep{generative_agents}, while subsequent work introduces
prospective and retrospective consolidation, scalable extraction
pipelines, or learned writing
policies~\citep{rmm,mem0,nemori}. Across these families, the
core RAG recipe remains the same---encode interaction history into a
retrievable store and surface relevant items on demand---but they differ
substantially in granularity (turn, session, fact, graph), update
discipline (online, batched, sleep-time), and retrieval signal
(embedding similarity, keywords, structural traversal).

\subsection{RL-based Memory Methods}

While most retrieval-augmented memory systems rely on hand-crafted
write/update/retrieve heuristics, a recent line of work casts memory
management itself as a sequential decision problem and trains the
controller with reinforcement learning so that storage and recall are
optimized for downstream task reward rather than surface-level
similarity~\citep{memory-r1,mem-alpha,memagent, rememr1, memcoe}.
A first family learns explicit \emph{memory-operation policies}: an LLM
agent emits structured actions such as \textsc{Add}, \textsc{Update},
\textsc{Delete}, and \textsc{NoOp}, and is fine-tuned with PPO or GRPO
using outcome rewards from QA accuracy, often coupled with a separate
RL-trained answering or retrieval agent~\citep{memory-r1, agemem}.
A second family targets \emph{end-to-end memory construction}: Mem-$\alpha$
trains a multi-component (core/episodic/semantic) memory with tool-style
operations and generalizes from short training contexts to sequences an
order of magnitude longer~\citep{mem-alpha}, while MemAgent extends RL
to multi-conversation training with DAPO-style updates so the agent can
operate on million-token histories~\citep{memagent}. A third strand
focuses on \emph{reward shaping and credit assignment}, since
memory rewards are typically sparse and trajectory-level:
ReMemR1 mixes final-answer rewards with dense step-level signals over
revisitable memory~\citep{rememr1}. These memory-centric RL methods are
closely related to, but distinct from, RL-trained retrieval/search
agents that learn when and how to query an external
corpus~\citep{search-r1, r1-searcher,research}, in that the policy must
additionally decide \emph{what to remember} and \emph{when to forget}.
Across the literature, the common message is that learned memory
controllers consistently outperform fixed heuristics on long-horizon
benchmarks such as LoCoMo, LongMemEval, and MSC.

\section{Prompts}

\clearpage

%%%% Prompt

\definecolor{promptbg}{HTML}{F5F7FA}
\definecolor{promptframe}{HTML}{3B6FA0}
\definecolor{prompttitle}{HTML}{FFFFFF}
 
\begin{figure*}[h]
\begin{promptbox}{Transfer Prompt for User-Assistant Conversation}
 
\textbf{System Role:} You are an advanced text transformation AI. Your task is to process a given conversation between a user and an assistant.
 
\medskip
\textbf{Rules:}
\begin{promptrules}
    \item Replace specific nouns (e.g., occupations, professional fields, locations, dates, names, times, and objects) with alternative but relevant ones, ensuring the topic of the conversation remains unchanged.
    \item Summarize both the user's and assistant's messages.
    \item Limit user's responses to 40 words.
    \item Limit assistant's responses to 80 words.
    \item You can rearrange the order of the sentences to maintain better coherence between them.
\end{promptrules}
 
\medskip
\textbf{Output Format:}
\begin{promptcode}
\ttfamily
\phl{[user]:}~~[content]\\
\phl{[assistant]:}~~[content]
\end{promptcode}
 
\medskip
\textbf{Input:} \{$text$\}
 
\end{promptbox}
\caption{Transfer Prompt for User-Assistant Conversation}
\label{transfer1}
\end{figure*}

\begin{figure*}[h]
\begin{promptbox}{Transfer Prompt for User-User Conversation}

\textbf{System Role:} You are an advanced text transformation AI. Your task is to process a given conversation between users.

\medskip
\textbf{Rules:}
\begin{promptrules}
    \item Replace specific nouns (e.g., names, occupations, professional fields, locations, dates, times, and objects) with alternative but relevant ones, ensuring the topic of the conversation remains unchanged.
    \item You may try rearranging the order of the sentences to maintain better coherence between them.
\end{promptrules}

\medskip
\textbf{Output Format:}
\begin{promptcode}
\ttfamily
\phl{[user name 1]:}~~[content]\\
\phl{[user name 2]:}~~[content]
\end{promptcode}

\medskip
\textbf{Input:} \{$text$\}

\end{promptbox}
\caption{Transfer Prompt for User-User Conversation}
\label{transfer2}
\end{figure*}

\begin{figure*}[h]
\begin{promptbox}{Episodic Memory Classification Prompt}

\textbf{Goal:} You are an advanced AI tasked with determining whether a specific piece of dialogue contains Episodic Memory. Use the following definition, guidelines, and examples to identify whether the provided dialogue qualifies as Episodic Memory.

\medskip
\textbf{Definition:} Episodic Memory refers to specific episodes or events that occurred (are planned to occur, or the thoughts on future plans) at a particular place and time. For the dialogue to qualify as Episodic Memory, it often includes details:
\begin{description}[
    leftmargin=4.8em,
    labelwidth=4.0em,
    labelsep=0.6em,
    align=right,
    font=\normalfont
]
    \item[\phl{When:}] A time, date, time frame, or vague contextual time expressions.
    \item[\phl{Where:}] A location or context.
    \item[\phl{Who:}] Participants.
    \item[\phl{What:}] Actions, or objects involved in the event.
\end{description}

\medskip
\textbf{Additional Guidelines:}
\begin{itemize}[leftmargin=1.5em, itemsep=2pt, label=\phl{\textbullet}]
    \item The time and location information does not always need to be precise. Vague or contextual expressions such as ``when I was driving,'' ``during lunch'', or ``while I was in school,'' can also qualify as time-related or location-related details if they provide sufficient information for the event.
    \item Questions that reveal or imply specific episodes or events information should also be classified as Episodic Memory.
    \item Missing one type of detail does not disqualify the dialogue from being Episodic Memory if the other details are sufficiently explicit to describe an instance.
    \item If the memory only describes general habits, vague knowledge, or patterns, it does not qualify as Episodic Memory.
    \item If the dialogue contains at least one instance of Episodic Memory, the output must be ``Episodic Memory'', regardless of whether other parts of the dialogue do not contain Episodic Memory.
\end{itemize}

\medskip
\textbf{Output Requirements:}
\begin{promptcode}
\ttfamily
If contains Episodic Memory \(\rightarrow\) \phl{Episodic Memory}\\
If does not contain Episodic Memory \(\rightarrow\) \phl{None}\\
\normalfont\textit{Do not include additional context, comments, symbols, or explanations.}
\end{promptcode}

\medskip
\textbf{Input:} \{$text\_to\_be\_classified$\}

\medskip
\textbf{Output:} Now, based on the above instructions, determine whether the input data contains Episodic Memory.

\end{promptbox}
\caption{Episodic Memory Classification Prompt}
\label{em_prompt}
\end{figure*}

\begin{figure*}[h]
\begin{promptbox}{Personal Semantic Memory Classification Prompt}

\textbf{Goal:} You are an advanced AI tasked with determining whether a segment of dialogue contains Personal Semantic Memory. Use the following definition, guidelines, and examples to identify whether the provided dialogue qualifies as Personal Semantic Memory.

\medskip
\textbf{Definition:} Personal Semantic Memory is a form of long-term memory consisting of knowledge specifically related to the person (user or another individual mentioned in the dialogue), such as descriptions of their information, preferences, habits, thoughts, background, emotions, education or work career-related information, and any other knowledge that reflects the individual's identity.

\medskip
\textbf{Additional Guidelines:}
\begin{itemize}[leftmargin=1.5em, itemsep=2pt, label=\phl{\textbullet}]
    \item Content that reveals personal information, preferences, habits, plans, background information, emotions, personal history, education or work career-related information, or any other knowledge that related to the individual's identity must be classified as Personal Semantic Memory.
    \item Statements that express the individual's intentions, work arrangements, or preliminary plans (not include details), even if they are inferred, belong to the category of Personal Semantic Memory.
    \item General knowledge unrelated to personal information, preferences, or emotions do not qualify as Personal Semantic Memory.
    \item Questions that reveal or imply personal details, preferences, or habits (user or another individual), should also be classified as Personal Semantic Memory.
    \item If the dialogue contains at least one instance of Personal Semantic Memory, the output must be ``Personal Semantic Memory'', regardless of whether other parts of the dialogue do not contain Personal Semantic Memory. The unanswered questions or requests in the conversation should not affect the final classification result.
\end{itemize}

\medskip
\textbf{Output Requirements:}
\begin{promptcode}
\ttfamily
If contains Personal Semantic Memory \(\rightarrow\) \phl{Personal Semantic Memory}\\
If does not contain Personal Semantic Memory \(\rightarrow\) \phl{None}\\
\normalfont\textit{Do not include additional context, comments, symbols, or explanations.}
\end{promptcode}

\medskip
\textbf{Input:} \{$text\_to\_be\_classified$\}

\medskip
\textbf{Output:} Now, based on the above instructions, determine whether the input data contains Personal Semantic Memory.

\end{promptbox}
\caption{Personal Semantic Memory Classification Prompt}
\label{pm_prompt}
\end{figure*}

\begin{figure*}[h]
\begin{promptbox}{General Semantic Memory Classification Prompt}

\textbf{Goal:} You are an advanced AI tasked with determining whether a given text contains General Semantic Memory based on the following definition, criteria, and examples.

\medskip
\textbf{Definition:} General Semantic Memory is a form of long-term memory consisting of non-personalized knowledge about the world, including facts, concepts, language, universally applicable methods, and objective descriptions of locations or objects. This includes information that is not tied to any specific individual or event. For information to qualify as General Semantic Memory, it must meet at least one of the following criteria:
\begin{itemize}[leftmargin=1.5em, itemsep=2pt, label=\phl{\textbullet}]
    \item \phl{Non-Personalized Knowledge or Broadly Applicable Suggestions:} It could include general knowledge, factual information, or widely applicable suggestions and methods. Examples: Scientific facts, historical events, cultural concepts, general knowledges, or general advice on common activities.
    \item \phl{Objective Descriptions of Locations, Objects, or Concepts:} Information that provides objective, factual details about locations, objects, or entities. Examples: Descriptions of landmarks, places, or any items.
\end{itemize}

\medskip
\textbf{Additional Guidelines:}
\begin{itemize}[leftmargin=1.5em, itemsep=2pt, label=\phl{\textbullet}]
    \item Focus on analyzing sentences that provide valid information. Dialogue that consists solely of questions without answers should not be used as a basis for General Semantic Memory classification.
    \item The information about a specific event or personalized to the user (e.g., user preferences, personal habits, or individual-specific details) does not qualify as General Semantic Memory.
    \item If the dialogue contains at least one instance of General Semantic Memory, the output must be ``General Semantic Memory'', regardless of whether other parts of the dialogue do not contain General Semantic Memory.
\end{itemize}

\medskip
\textbf{Output Requirements:}
\begin{promptcode}
\ttfamily
If contains General Semantic Memory \(\rightarrow\) \phl{General Semantic Memory}\\
If does not contain General Semantic Memory \(\rightarrow\) \phl{None}\\
\normalfont\textit{Do not include additional context, comments, symbols, or explanations.}
\end{promptcode}

\medskip
\textbf{Input:} \{$text\_to\_be\_classified$\}

\medskip
\textbf{Output:} Now, based on the above instructions, determine whether the input data contains General Semantic Memory.

\end{promptbox}
\caption{General Semantic Memory Classification Prompt}
\label{gm_prompt}
\end{figure*}

\begin{figure*}[h]
\begin{promptbox}{Memory Multi-Classification Prompt}

\textbf{Goal:} You are an advanced AI tasked with classifying a dialogue between a user and a chatbot into one or more of the following memory types: ``Episodic Memory,'' ``Personal Semantic Memory,'' or ``General Semantic Memory.'' Use the following definitions to guide your classification:
\begin{itemize}[leftmargin=1.5em, itemsep=2pt, label=\phl{\textbullet}]
    \item \phl{Episodic Memory:} A form of long-term memory consisting of non-personalized knowledge about the world, including facts, concepts, language, universally applicable methods, and objective descriptions of locations or objects. This includes information that is not tied to any specific individual or event.
    \item \phl{Personal Semantic Memory:} A form of long-term memory consisting of knowledge specifically related to the person (user or another individual mentioned in the dialogue), such as descriptions of their information, preferences, habits, thoughts, background, emotions, education or work career-related information, and any other knowledge that reflects the individual's identity.
    \item \phl{General Semantic Memory:} A form of long-term memory consisting of general knowledge about the world, including facts, concepts, language, universally applicable methods, and objective description. This includes information that is broadly relevant, commonly recognized, or widely recommended.
\end{itemize}

\medskip
\textbf{Output Requirements:} Your output must strictly adhere to one or more of the following memory types, separated by commas if multiple apply:
\begin{promptcode}
\ttfamily
\phl{Episodic Memory}, \phl{Personal Semantic Memory}, \phl{General Semantic Memory}\\[2pt]
If none of the above types apply \(\rightarrow\) \phl{None}\\[2pt]
\normalfont\textit{Do not include any additional words or context beyond the specified output format.}
\end{promptcode}

\noindent Carefully analyze the dialogue and determine the appropriate memory type(s) based on the content and context provided.

\medskip
\textbf{Input:} \{$text\_to\_be\_classified$\}

\medskip
\textbf{Output:} Now, classify the memory type(s) of the input data based on the instructions above.

\end{promptbox}
\caption{Memory Multi-Classification Prompt}
\label{multi_class_prompt}
\end{figure*}

\begin{figure*}[h]
\begin{promptbox}{HyDE Prompt}

Please write a paragraph that answers the question.

\medskip
\textbf{Question:} \{$query$\}

\medskip
\textbf{Output:}

\end{promptbox}
\caption{HyDE Prompt}
\label{hyde_p}
\end{figure*}

\begin{figure*}[h]
\begin{promptbox}{MILL Prompt}

What sub-queries should be searched to answer the following query? Please generate 5 sub-queries with their related passages.

\medskip
\textbf{Question:} \{$query$\}

\medskip
\textit{Only present the subquestion, without any other words and explanation.}

\end{promptbox}
\caption{MILL Prompt}
\label{mill_p}
\end{figure*}

\begin{figure*}[h]
\begin{promptbox}{Query2Doc Prompt}

Write a passage that answers the given query:

\medskip
\textbf{Question:} \{$query$\}

\medskip
\textbf{Output:}

\end{promptbox}
\caption{Query2Doc Prompt}
\label{query2doc_p}
\end{figure*}

\begin{figure*}[h]
\begin{promptbox}{Answer Generation Prompt for User-Assistant Conversation}

\textbf{Goal:} Suppose a user has recent conversation records with assistant. Use your imagination to generate conversation records. The generated conversation records must contain information related to the given query. The conversation must be logically clear and structurally reasonable, representing a discussion on a specific topic rather than a direct recollection of the query. Conversation record should be generated from one of the following perspectives:
\begin{itemize}[leftmargin=1.5em, itemsep=2pt, label=\phl{\textbullet}]
    \item When the query is not an advice-seeking type of question, a sentence can be output that directly provides an answer matching the query, avoiding expressions other than stating the answer.
    \item Instead of providing a direct answer, it can describe the user's preferences, habits, events, or background related to the topic of the query. This is not a direct answer to the query, but should be an additional statement from the user, especially for advice-seeking type queries.
\end{itemize}

\medskip
\textbf{Input Data:} The input is a natural language query posed by the user, typically related to the previous conversations information.

\medskip
\textbf{Output Requirements:}
\begin{itemize}[leftmargin=1.5em, itemsep=2pt, label=\phl{\textbullet}]
    \item The output must not exceed one sentence. You should determine whether the query-related information is mentioned by the user or the assistant.
    \item Ensure that the response contains only the dialogue content of one speaker.
\end{itemize}

\medskip
\textbf{Input:} \{$query\_to\_be\_answered$\}

\medskip
\textbf{Output:} Generate a hypothetical dialogue record.

\end{promptbox}
\caption{Answer Generation Prompt for User-Assistant Conversation}
\label{fake_memory_p1}
\end{figure*}
\begin{figure*}[h]
\begin{promptbox}{Answer Generation Prompt for User-User Conversation}

\textbf{Goal:} Suppose there is a conversation between users. Imagine you are the user mentioned in the query and generate a dialogue in the first person. The generated dialogue record must contain relevant information that answers the given query. The memory must be logical, well-structured, and tailored to the type of question provided.

\medskip
\textbf{Input Data:} The input is a query in natural language that asks a specific question or seeks information.

\medskip
\textbf{Output Requirements:}
\begin{itemize}[leftmargin=1.5em, itemsep=2pt, label=\phl{\textbullet}]
    \item \phl{Completeness:} Ensure that each dialogue memory provides sufficient information to address the query comprehensively.
    \item \phl{Length:} The output must not exceed one to two sentences. Ensure brevity while maintaining clarity and relevance.
    \item \phl{Perspective:} The response should be written from the perspective of the user mentioned in the query, as a single speaker.
\end{itemize}

\medskip
\textbf{Output Format:}
\begin{promptcode}
\ttfamily
\phl{[user name in query]:}~~[response content]
\end{promptcode}

\medskip
\textbf{Input:} \{$query\_to\_be\_answered$\}

\medskip
\textbf{Output:} Generate a hypothetical dialogue record.

\end{promptbox}
\caption{Answer Generation Prompt for User-User Conversation}
\label{fake_memory_p2}
\end{figure*}

\begin{figure*}[h]
\begin{promptbox}{Query Event Extract Prompt}

\textbf{Instruction:} Identify and extract the elements of events mentioned in the text below. The text was recorded on \textbf{\{session\_date\}}. If the dataset does not record a session date this field is filled with the literal string ``unknown''; in that case skip any relative-time resolution.

The text may be either (a)~a turn from a chat between a user and an AI assistant, or (b)~a piece of structured personal-knowledge content (profile, social relationships, education, profession, \ldots). Many inputs do \textbf{not} describe a real event at all (chit-chat, opinions, task requests, stable attributes); for those cases follow the strict N/A rules below.

\medskip
\textbf{Output Format:} Use exactly the following structured format. Repeat the block once per distinct event, separated by a blank line. If nothing event-like is present, emit one block with all three fields = `N/A'.
\begin{promptcode}
\ttfamily
{[Event n]}\\
\phl{Time:}~~[Absolute date in YYYY-MM-DD, YYYY-MM, YYYY, or `N/A']\\
\phl{Person(s):}~~[Person(s) involved, or `N/A']\\
\phl{Location:}~~[Location, or `N/A']
\end{promptcode}

\medskip
\textbf{When to Output Everything `N/A':} Emit a single all-N/A block when the input is:
\begin{itemize}[leftmargin=1.5em, itemsep=2pt, label=\phl{\textbullet}]
    \item A question to the assistant (``How do I \ldots?'', ``What is \ldots?'')
    \item A request to the assistant (``Write me a poem about \ldots'')
    \item A generic opinion or feeling not tied to an event (``I think Python is great'')
    \item A stable preference or habit (``I usually drink coffee in the morning'')
    \item A hypothetical / future plan with no anchored time
\end{itemize}

\medskip
\textbf{Time --- Strict Rules:}
\begin{itemize}[leftmargin=1.5em, itemsep=2pt, label=\phl{\textbullet}]
    \item Resolve every relative time expression (``yesterday'', ``last week'', ``two days ago'', ``this morning'', ``next month'', \ldots) against \{session\_date\} and output the absolute date. Never keep the relative phrase in Time.
    \item If only year + month is clear, output \texttt{YYYY-MM}; year only, output \texttt{YYYY}; otherwise `N/A'.
\end{itemize}

\medskip
\textbf{Person(s) --- Strict Rules:}
\begin{itemize}[leftmargin=1.5em, itemsep=2pt, label=\phl{\textbullet}]
    \item Replace every literal ``User''/``user'' with ``I''.
    \item Never include ``Assistant'', ``AI'', ``Bot'', ``ChatGPT'', ``the model'', or any other non-human dialog role.
    \item Copy named third persons verbatim. Never invent names.
\end{itemize}

\medskip
\textbf{Location --- Strict Rules:} Fill only when an explicit place is mentioned in the text.

\medskip
\textbf{Conversation Date:} \{$session\_date$\}

\medskip
\textbf{Input Data:} \{$text\_to\_be\_processed$\}

\end{promptbox}
\caption{Query Event Extract Prompt}
\label{query_event_p}
\end{figure*}

\begin{figure*}[h]
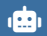

\begin{promptbox}{Conversation Event Extract Prompt}

\textbf{Instruction:} Identify and extract the event- or attribute-related elements mentioned in a query. The query is asked on \textbf{\{question\_date\}}. If the dataset does not record a query date this field is filled with the literal string ``unknown''; in that case skip any relative-time resolution and treat Time accordingly.

A query may target either (a)~a concrete event referenced in the dialog (a trip, a graduation, an action the user took, a plan the user has), or (b)~a stable personal attribute about the speaker or a third person named in the query (occupation, hobby, residence, relationship, etc.). Both cases are handled with the same four output fields below.

\medskip
\textbf{Output Format} (exactly four lines, in this order):
\begin{promptcode}
\ttfamily
\phl{Time:}~~[Absolute date in YYYY-MM-DD, YYYY-MM, YYYY, or `N/A']\\
\phl{Person(s):}~~[Person(s) involved, or `N/A']\\
\phl{Location:}~~[Location, or `N/A']\\
\phl{Event:}~~[Concise event or attribute description]
\end{promptcode}

\medskip
\textbf{Time --- Strict Rules:}
\begin{itemize}[leftmargin=1.5em, itemsep=2pt, label=\phl{\textbullet}]
    \item Resolve every relative time expression (``yesterday'', ``last week'', ``two days ago'', ``this morning'', \ldots) against \{question\_date\} and output the absolute date. Never keep a relative phrase in Time.
    \item Output `N/A' if the query, combined with the question date, does not pin down a specific date / month / year.
    \item Never invent a date that is not stated or implied by the query.
    \item For pure attribute questions (e.g., ``What is X's occupation?''), Time is almost always `N/A'.
\end{itemize}

\medskip
\textbf{Query Date:} \{$question\_date$\}

\medskip
\textbf{Input Data:} \{$text\_to\_be\_processed$\}

\end{promptbox}
\caption{Conversation Event Extract Prompt}
\label{conversation_event_p}
\end{figure*}

\begin{figure*}[h]
\begin{promptbox}{Conversation Keywords Expansion Prompt}

\textbf{Instruction:} Please provide additional search keywords for each of the key aspects of the following queries that make it easier to find the relevant documents. Do not include irrelevant text and separate the search topics with commas.

\medskip
\textbf{Input:} Query is: \{$query$\}

\end{promptbox}
\caption{Conversation Keywords Expansion Prompt}
\label{keyword_p}
\end{figure*}

\begin{figure*}[h]
\begin{promptbox}{Memory Pruning Prompt}

\textbf{Role:} You are a dialog assistant given a numbered, dated list of memories and a user query. You will (a)~answer the query using only information that can be supported by the memories, and (b)~explicitly cite the memory indices that you actually used.

\medskip
\textbf{Inputs:}
\begin{itemize}[leftmargin=1.5em, itemsep=2pt, label=\phl{\textbullet}]
    \item \phl{Query date:} The date on which the user asks the question. May be the literal string ``unknown'' if the dataset does not record a query date; in that case skip any temporal reasoning.
    \item \phl{Question:} The user query.
    \item \phl{Memory:} A numbered list. Each item starts with the header \texttt{[Memory <i>] (session\_date=<date>)} followed by the original dialog text. \texttt{session\_date} may also be ``unknown''.
\end{itemize}

\medskip
\textbf{Procedure:}
\begin{promptrules}
    \item Read every memory carefully. When the question is time-sensitive (e.g., ``last week'', ``yesterday'', ``since May''), use the per-memory \texttt{session\_date} together with the Query date to ground relative temporal expressions before deciding which memories support the answer.
    \item Produce a single concise answer using ONLY information that is supported by the memories. If no memory supports the question, answer ``I don't know'' rather than fabricating.
    \item List the 1-based memory indices that you actually relied on, taken verbatim from the \texttt{[Memory <i>]} headers.
\end{promptrules}

\medskip
\textbf{Output Format} (exactly two lines, nothing else):
\begin{promptcode}
\ttfamily
\phl{Answer:}~~<your concise answer>\\
\phl{References:}~~[i, j, k]
\end{promptcode}

\medskip
\textbf{Rules for References:}
\begin{itemize}[leftmargin=1.5em, itemsep=2pt, label=\phl{\textbullet}]
    \item Use a JSON-style list of integers, e.g., \texttt{References: [1, 4]} or \texttt{References: []} when no memory supports the answer.
    \item Never invent an index that is not present in the Memory list.
    \item Do not include duplicate indices and do not wrap them in quotes.
\end{itemize}

\medskip
\textbf{Query date:} \{$query\_date$\}

\textbf{Question:} \{$question$\}

\textbf{Memory:} \{$context$\}

\end{promptbox}
\caption{Memory Pruning Prompt}
\label{pruning_p}
\end{figure*}

\begin{figure*}[h]
\begin{promptbox}{Response Prompt}

\textbf{Role:} You are an intelligent dialog bot. You are shown a small set of Related Evidences (each annotated with a \texttt{session\_date}) and the user's query (annotated with a Query date). The evidences have already been filtered down by an earlier reference-selection step, so every item below is expected to be relevant.

\medskip
Read, understand, and synthesize the evidences, then output ONE concise, coherent and helpful answer. Provide the answer text directly, without prefixes, prologues, or any extra commentary.

\medskip
\textbf{Query date:} \{$query\_date$\}

\textbf{Context:} \{$context$\}

\textbf{Question:} \{$question$\}

\end{promptbox}
\caption{Response Prompt, which follows \citep{memochat,secom}}
\label{response_p}
\end{figure*}

\begin{figure*}[h]
\begin{promptbox}{GPT Judge Prompt}

I will give you a question, a reference answer, and a response from a model. Please answer [[yes]] if the response contains the reference answer. Otherwise, answer [[no]].

If the response is equivalent to the correct answer or contains all the intermediate steps to get the reference answer, you should also answer [[yes]]. If the response only contains a subset of the information required by the answer, answer [[no]].

\medskip
\begin{promptcode}
\ttfamily
\phl{[User Question]}\\
\{$question$\}\\[4pt]
\phl{[The Start of Reference Answer]}\\
\{$answer$\}\\
\phl{[The End of Reference Answer]}\\[4pt]
\phl{[The Start of Model's Response]}\\
\{$response$\}\\
\phl{[The End of Model's Response]}
\end{promptcode}

\medskip
Is the model response correct? Answer \texttt{[[yes]]} or \texttt{[[no]]} only.

\end{promptbox}
\caption{GPT Judge Prompt, which follows \citep{longmemeval,zheng2023judging}}
\label{g4j_p}
\end{figure*}

\clearpage

%%%%%%%%%%%%%%%%%%%%%%%%%%%%%%%%%%%%%%%%%%%%%%%%%%%%%%%%%%%%

\newpage
\section*{NeurIPS Paper Checklist}

\begin{enumerate}

\item {\bf Claims}
    \item[] Question: Do the main claims made in the abstract and introduction accurately reflect the paper's contributions and scope?
    \item[] Answer: \answerYes{} % Replace by \answerYes{}, \answerNo{}, or \answerNA{}.
    \item[] Justification: The main claims are stated in the abstract and introduction

\item {\bf Limitations}
    \item[] Question: Does the paper discuss the limitations of the work performed by the authors?
    \item[] Answer: \answerYes{} % Replace by \answerYes{}, \answerNo{}, or \answerNA{}.
    \item[] Justification: Please refer to Section 6.

\item {\bf Theory assumptions and proofs}
    \item[] Question: For each theoretical result, does the paper provide the full set of assumptions and a complete (and correct) proof?
    \item[] Answer:  \answerYes{} % Replace by \answerYes{}, \answerNo{}, or \answerNA{}.
    \item[] Justification: We have given the full assumptions and complete proofs.

    \item {\bf Experimental result reproducibility}
    \item[] Question: Does the paper fully disclose all the information needed to reproduce the main experimental results of the paper to the extent that it affects the main claims and/or conclusions of the paper (regardless of whether the code and data are provided or not)?
    \item[] Answer: \answerYes{} % Replace by \answerYes{}, \answerNo{}, or \answerNA{}.
    \item[] Justification: We have included all the information needed to reproduce the results.

\item {\bf Open access to data and code}
    \item[] Question: Does the paper provide open access to the data and code, with sufficient instructions to faithfully reproduce the main experimental results, as described in supplemental material?
    \item[] Answer: \answerYes{} % Replace by \answerYes{}, \answerNo{}, or \answerNA{}.
    \item[] Justification: We have released the code. In addition, the datasets used in our work are all public datasets.

\item {\bf Experimental setting/details}
    \item[] Question: Does the paper specify all the training and test details (e.g., data splits, hyperparameters, how they were chosen, type of optimizer) necessary to understand the results?
    \item[] Answer: \answerYes{} % Replace by \answerYes{}, \answerNo{}, or \answerNA{}.
    \item[] Justification:  Please refer to the Implementation Details Section.

\item {\bf Experiment statistical significance}
    \item[] Question: Does the paper report error bars suitably and correctly defined or other appropriate information about the statistical significance of the experiments?
    \item[] Answer: \answerYes{} % Replace by \answerYes{}, \answerNo{}, or \answerNA{}.
    \item[] Justification: We conduct the two-sided t-test on our results. 

\item {\bf Experiments compute resources}
    \item[] Question: For each experiment, does the paper provide sufficient information on the computer resources (type of compute workers, memory, time of execution) needed to reproduce the experiments?
    \item[] Answer: \answerYes{} % Replace by \answerYes{}, \answerNo{}, or \answerNA{}.
    \item[] Justification: Please refer to Implementation Details.
    
\item {\bf Code of ethics}
    \item[] Question: Does the research conducted in the paper conform, in every respect, with the NeurIPS Code of Ethics \url{https://neurips.cc/public/EthicsGuidelines}?
    \item[] Answer: \answerYes{} % Replace by \answerYes{}, \answerNo{}, or \answerNA{}.
    \item[] Justification: We conform to the NeurIPS Code of Ethics.

\item {\bf Broader impacts}
    \item[] Question: Does the paper discuss both potential positive societal impacts and negative societal impacts of the work performed?
    \item[] Answer: \answerYes{} % Replace by \answerYes{}, \answerNo{}, or \answerNA{}.
    \item[] Justification: We discuss the impacts in the Introduction section.
    
\item {\bf Safeguards}
    \item[] Question: Does the paper describe safeguards that have been put in place for responsible release of data or models that have a high risk for misuse (e.g., pre-trained language models, image generators, or scraped datasets)?
    \item[] Answer: \answerNA{} % Replace by \answerYes{}, \answerNo{}, or \answerNA{}.
    \item[] Justification: All datasets used in our work are public datasets.

\item {\bf Licenses for existing assets}
    \item[] Question: Are the creators or original owners of assets (e.g., code, data, models), used in the paper, properly credited and are the license and terms of use explicitly mentioned and properly respected?
    \item[] Answer: \answerYes{} % Replace by \answerYes{}, \answerNo{}, or \answerNA{}.
    \item[] Justification: We properly cite the dataset and models used in our work.

\item {\bf New assets}
    \item[] Question: Are new assets introduced in the paper well documented and is the documentation provided alongside the assets?
    \item[] Answer: \answerYes{} % Replace by \answerYes{}, \answerNo{}, or \answerNA{}.
    \item[] Justification: We have given the descriptions in our paper.

\item {\bf Crowdsourcing and research with human subjects}
    \item[] Question: For crowdsourcing experiments and research with human subjects, does the paper include the full text of instructions given to participants and screenshots, if applicable, as well as details about compensation (if any)? 
    \item[] Answer: \answerNA{} % Replace by \answerYes{}, \answerNo{}, or \answerNA{}.
    \item[] Justification: No crowdsourcing experiments are conducted in this work.

\item {\bf Institutional review board (IRB) approvals or equivalent for research with human subjects}
    \item[] Question: Does the paper describe potential risks incurred by study participants, whether such risks were disclosed to the subjects, and whether Institutional Review Board (IRB) approvals (or an equivalent approval/review based on the requirements of your country or institution) were obtained?
    \item[] Answer: \answerNA{} % Replace by \answerYes{}, \answerNo{}, or \answerNA{}.
    \item[] Justification: No potential risk is found in this work.

\item {\bf Declaration of LLM usage}
    \item[] Question: Does the paper describe the usage of LLMs if it is an important, original, or non-standard component of the core methods in this research? Note that if the LLM is used only for writing, editing, or formatting purposes and does \emph{not} impact the core methodology, scientific rigor, or originality of the research, declaration is not required.
    %this research? 
    \item[] Answer: \answerNA{} % Replace by \answerYes{}, \answerNo{}, or \answerNA{}.
    \item[] Justification: LLM usage does not impact the core methodology, scientific rigorousness, or originality of the research.

\end{enumerate}

\end{document}